%% file: main.tex
\documentclass[11pt]{article}
\usepackage[margin=1in]{geometry}
\usepackage{amsmath,amssymb,amsthm}
\usepackage{booktabs}
\usepackage{graphicx}
\usepackage{microtype}
\usepackage{xcolor}
\usepackage{tikz}
\usepackage{longtable}
\usetikzlibrary{arrows.meta,positioning,fit,backgrounds,shapes.geometric}
\definecolor{cfront}{HTML}{1B4F72}
\definecolor{cmid}{HTML}{2E86C1}
\definecolor{ccheap}{HTML}{AED6F1}
\definecolor{cceil}{HTML}{922B21}
\definecolor{cgain}{HTML}{1E8449}
\definecolor{cbg}{HTML}{F4F6F7}
\usepackage[numbers,sort&compress]{natbib}
\usepackage[hidelinks]{hyperref}

\theoremstyle{plain}
\newtheorem{proposition}{Proposition}
\newtheorem{lemma}{Lemma}
\newtheorem{corollary}{Corollary}
\theoremstyle{definition}

\newcommand{\E}{\mathbb{E}}
\newcommand{\Var}{\operatorname{Var}}
\newcommand{\Corr}{\operatorname{Corr}}

\title{\bfseries When Does Combining Language Models Help?\\[4pt]
{\large\mdseries A Co-Failure Ceiling on Routing, Voting, and Mixture-of-Agents Across 67 Frontier Models}}
\author{Josef Chen\\ KAIKAKU\\ \texttt{josef@kaikaku.ai}}
\date{\today}
\begin{document}
\maketitle

\begin{abstract}
Multi-model LLM systems, including routing, voting, cascades, fusion, and mixture-of-agents (MoA), are increasingly
used to push accuracy beyond any single model. We show that the achievable gain has a ceiling, fixed by a quantity
other than the one the field usually reports. For any policy whose output is one of the member models' answers, such
as a router, a majority vote, or a cascade, accuracy cannot exceed $1-\beta$, where $\beta$ is the rate at which
every model is wrong on the same query. Practice instead reports the average pairwise error correlation $\rho$, and
we prove that $\rho$ cannot identify $\beta$: error laws with identical marginals and identical pairwise correlations
can still differ in $\beta$. A Clopper--Pearson bound on $\beta$ turns one graded, held-out query set into a \$0
certificate on the largest gain any such policy could deliver, before a router is trained. Across $67$ frontier
models from $21$ providers, among them GPT-5.5, Claude Opus 4.8, Gemini 3.1 Pro, Grok-4.3, DeepSeek V4,
Qwen3.7-Max, and Kimi K2.7, a correctly (tetrachoric) calibrated single-factor model still underprices the all-wrong
tail, by a margin that widens with the size of the pool: about $2.5$ times on open-ended mathematics ($90\%$ CI
$1.7$ to $3.4$, $k=17$), holding under the full $67$-model Gaussian copula ($\beta=0.052$ versus a predicted
$0.023$) and recurring on execution-graded code ($\beta=0.079$). At matched quality, a diverse low-$\rho$ ensemble
beats a high-$\rho$ Self-MoA one. Re-asking the same GPQA-Diamond questions in free-response rather than
multiple-choice form reopens the tail ($\beta=0.127$; a five-judge LLM panel, $\kappa$ from $0.73$ to $0.92$),
locating the effect in task format rather than subject. On our pool, and on tasks where answers can be checked,
combining models rarely beats the single best model without a strong query-level routing signal; the gains come from
models that fail on different questions rather than from adding more of them.
\end{abstract}

\section{Introduction}
The single-model era is closing. Enterprises serve production traffic across hundreds of models from dozens of
providers, selecting per workload on cost, latency, reliability, and capability; a routing layer increasingly
mediates spend and governance, and carries the provider risk \citep{ong2024routellm,hu2024routerbench,chen2023frugalgpt}. The
operative question has shifted from which model is best to \emph{how a buyer should allocate a token and
dollar budget across a heterogeneous, correlated, rapidly depreciating pool}. Practitioners answer it with a single
diagnostic: the pairwise error correlation $\rho$ between models, low values signalling that diversity will pay.

Our central finding is that this diagnostic is the wrong one. What bounds orchestration is $\beta$, the rate at which
\emph{all} models fail on the same query: no router, vote, or cascade can exceed accuracy $1-\beta$, and $\rho$ cannot
see $\beta$. The gap is not academic, because on today's open-ended tasks the strongest models increasingly fail
together. A Clopper--Pearson bound on $\beta$, measurable from one graded query set, says in advance how much room any
such policy has to beat the single best model.

\paragraph{What we concede.} The equicorrelated variance floor is classical portfolio and ensemble theory
\citep{markowitz1952,statman1987,krogh1995ensemble,ueda1996,wood2023diversity} and, in its Gaussian-copula form for
language-model ensembles, \citet{turkmen2026ensembleselection}; the oracle upper envelope and the optimality of routing
and cascading are due to \citet{dekoninck2024cascaderouting}; and our tools (linear-programming duality
\citep{bertsimas1997lp}, Clopper--Pearson intervals, the Gaussian copula, the single-factor probit) are standard. We
claim no new routing algorithm. The contribution is their specialization to priced inference orchestration and the
market-scale measurement.

\paragraph{Contributions.}
\begin{enumerate}\itemsep2pt
\item \emph{The orchestration ceiling and a finite-sample certificate} (\S\ref{sec:realiz}, Prop.~\ref{prop:cert}): no router,
vote, or cascade can exceed $1-\beta$; the oracle gain localizes as $\Pr[\text{single-best wrong}]-\beta$; and a Clopper--Pearson
bound turns one query sample into a certificate on the largest gain any such policy can deliver.
\item \emph{Why pairwise $\rho$ underprices co-failure} (Prop.~\ref{prop:poolbias}): under tail dependence the single-factor
estimate of $\beta$ from $\rho$ is downward-biased, with the bias diverging in pool size and driven by a common-mode atom rather
than tail dependence as such.
\item \emph{A market-scale measurement} (\S\ref{sec:setup}--\ref{sec:results}): on $67$ models from $21$ provider families,
oracle routing gain is positive yet a learned router realizes almost none of it; the $\beta$/$\rho$ gap and its growth with pool
size are measured directly; and two regimes, ceiling-bound (open-ended math) and realizability-bound (science), appear across
domains, though the decisive all-models-wrong counts are small (\S\ref{sec:results}).
\item \emph{Supporting economics} (App.~\ref{app:econ}): budget-constrained routing as a priced assignment with a single shadow
price (Prop.~\ref{prop:duality}); a cost-aware diversification limit (Props.~\ref{prop:keven},~\ref{prop:lambda}); and cascade
calibration boundaries (Prop.~\ref{prop:cascade}, Cor.~\ref{cor:edge}). These specialize standard tools and are deferred to the
appendix, as is an observational option value of breadth under churn (App.~\ref{app:churn}).
\end{enumerate}

\begin{figure}[t]\centering
\resizebox{\textwidth}{!}{%
\begin{tikzpicture}[font=\small,
  box/.style={rounded corners=2pt, draw, align=center, inner sep=4pt},
  model/.style={rounded corners=1pt, draw=none, text=white, font=\scriptsize\bfseries, minimum width=30mm, minimum height=4.4mm, inner sep=1pt},
  arr/.style={-{Stealth[length=2.4mm]}, thick, draw=cfront}]
\node[box, fill=cbg] (q) {query $x$\\[-1pt]\scriptsize type $t$};
\node[box, draw=cfront, very thick, fill=white, right=8mm of q, align=center] (orch)
  {\textbf{Orchestration}\\[1pt]\scriptsize route $\,\cdot\,$ cascade $\,\cdot\,$ fuse\\[1pt]\scriptsize shadow price $\lambda_B$};
\node[model, fill=cfront, right=12mm of orch, yshift=11mm] (m1) {GPT-5.5 \;\$30 \,$\cdot$\, Claude Opus 4.8};
\node[model, fill=cfront, below=1.2mm of m1] (m2) {GPT-5.4 \,$\cdot$\, Gemini 3.1 Pro \,$\cdot$\, Grok-4.3};
\node[model, fill=cmid, below=1.2mm of m2] (m3) {Qwen3.7-Max \,$\cdot$\, GLM-5.2 \,$\cdot$\, Kimi K2.7};
\node[model, fill=cmid, below=1.2mm of m3] (m4) {DeepSeek V4 \,$\cdot$\, MiniMax M3 \,$\cdot$\, Mistral};
\node[model, fill=ccheap, text=black, below=1.2mm of m4] (m5) {Llama-4 \,$\cdot$\, Gemma \,$\cdot$\, Phi \;\$0.1};
\begin{scope}[on background layer]
\node[draw=gray, dashed, rounded corners, fit=(m1)(m5), inner sep=5pt] (pool) {};
\end{scope}
\node[above=0.5mm of pool, font=\scriptsize, text=gray!50!black] {67 models, 21 families: correlated, priced, churning};
\node[box, fill=cgain!12, draw=cgain, right=10mm of pool, align=center] (out) {accuracy\\[-1pt]$\le 1-\beta$};
\draw[arr] (q) -- (orch);
\draw[arr] (orch) -- (pool.west);
\draw[arr] (pool.east) -- (out);
\node[box, fill=cceil!8, draw=cceil, below=9mm of orch, xshift=-2mm, align=left, font=\scriptsize, text width=52mm] (r1)
  {\textbf{\textcolor{cceil}{Ceiling-bound}} (open-ended math): $\beta>0$ caps every policy; a correctly-calibrated single-factor model still underprices the tail $\sim\!2.5\times$ (residual common mode), growing with pool size.};
\node[box, fill=cgain!8, draw=cgain, right=5mm of r1, align=left, font=\scriptsize, text width=52mm] (r2)
  {\textbf{\textcolor{cgain}{Realizability-bound}} (science): $\beta\!\approx\!0$, ceiling open; the $0.15$ oracle gain is resolvable disagreement a router would still have to realize.};
\end{tikzpicture}}
\caption{\label{fig:framework}\textbf{Orchestration is allocation.} A query is routed, cascaded, or fused over a priced,
correlated, fast-churning pool of $67$ frontier-to-cheap models; the optimal policy is a per-type bang-per-buck rule with a
single shadow price $\lambda_B$ on the inference dollar (Prop.~\ref{prop:duality}). No selection policy can exceed the ceiling
$1-\beta$ set by the rate $\beta$ at which \emph{all} models fail at once (Prop.~\ref{prop:cert}). The field reports pairwise
error correlation $\rho$ to decide whether to orchestrate; $\rho$ is blind to $\beta$, and headroom is foreclosed in two opposite
regimes that $\rho$ cannot tell apart.}
\end{figure}
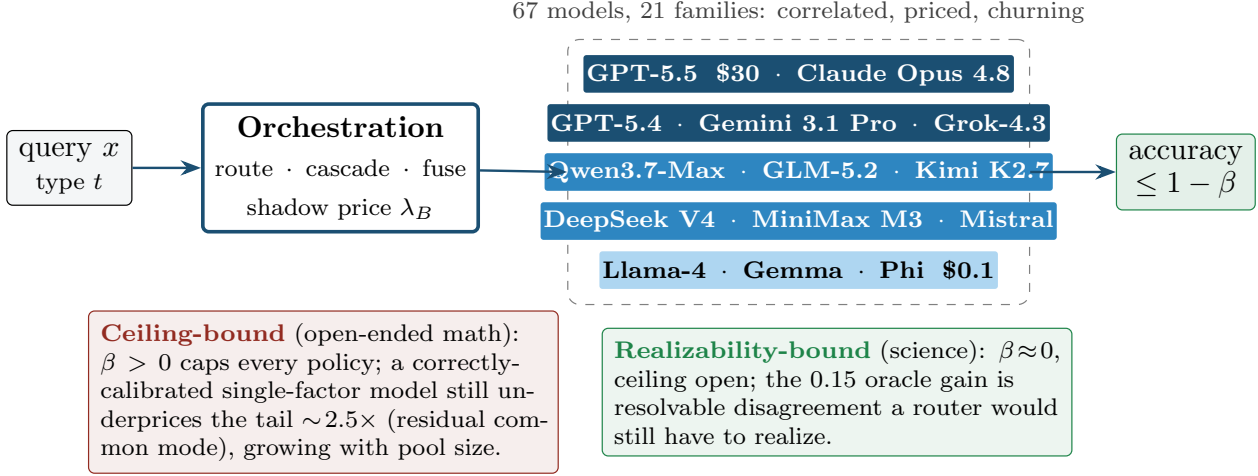

\section{Related Work}\label{sec:related}
\paragraph{Routing and cascades.} Learned routers select one model per query \citep{ong2024routellm,ding2024hybridllm};
cascades escalate from cheap to strong on low confidence \citep{chen2023frugalgpt,madaan2023automix}.
\citet{dekoninck2024cascaderouting} unify the two and prove optimality of a routing strategy and an optimal cascade;
\citet{jitkrittum2023cascade} characterize the optimal two-model deferral rule and show confidence-based deferral can
be sub-optimal when the downstream model's errors are unmodeled, directly relevant to our calibration-and-edge
condition. \citet{hu2024routerbench} benchmark routing on pre-computed outcomes. These works treat the oracle as an
empirical ceiling and the cascade threshold as a tuned hyperparameter; we add the dollar stopping rule, the
calibration dominance boundary, and the volume ceiling they leave implicit.
\paragraph{Ensembling and fusion.} \citet{jiang2023llmblender} rank-and-fuse generations; \citet{wang2024moa} aggregate
in layers; \citet{li2024moreagents} sample-and-vote; \citet{li2025selfmoa} show that sampling the single best model
(Self-MoA) often beats heterogeneous fusion, attributing this to a quality--diversity trade-off (mixing helps only when
members are of similar quality). Classical ensemble theory supplies the variance decomposition
\citep{krogh1995ensemble,ueda1996}, its modern unification \citep{wood2023diversity}, and the caution that diversity does
not guarantee accuracy gains \citep{kuncheva2003}. The accuracy \emph{ceiling} we use is also classical: it is Kuncheva's
\emph{Oracle} combiner (correct iff some member is correct, i.e.\ $1-\beta$ \citep{kuncheva2004combining}) and the
majority-vote-accuracy limits of \citet{kuncheva2003limits}, both predating LLMs and presuming labels. Our contribution is
not the ceiling but (i) that it bounds \emph{every} selection policy, debate and self-consistency included, whose outputs
are a.s.\ member answers, and (ii) its conversion into a labels-free, finite-sample \$0 \emph{certificate} (Prop.~\ref{prop:cert});
an Oracle needs the labels a certificate does not.
\paragraph{Error correlation and inference economics.} \citet{kim2025correlated} document, over $350+$ models, that
\emph{pairwise} errors are substantially correlated and rise with accuracy and shared provider, but they measure the
bivariate agree-when-wrong statistic on multiple-choice tasks and stop at qualitative implications; they neither define the
joint all-models-wrong rate $\beta$ nor bound any selection policy by it. We show that pairwise correlation provably cannot
identify $\beta$ for $m\ge3$ (Prop.~\ref{prop:nonid}), and that $\beta$, not $\rho$, is the ceiling. This monoculture of
errors traces to \citet{kleinberg2021monoculture}. Closest is \citet{turkmen2026ensembleselection}, who also fit a
tetrachoric Gaussian copula to binary LLM errors and derive an equicorrelated ensemble error \emph{floor}; we reach the
\emph{opposite} conclusion and show why: that very Gaussian floor \emph{underprices} the empirical co-failure tail, because
the driver is a common-mode atom the copula cannot represent (zero lower tail dependence; Prop.~\ref{prop:poolbias}), which we
confirm on the full $67$-model matrix and against a Clayton control. The copula they use to bound the tail is the one we
measure to be loose. \citet{erol2025costofpass} formalize dollars-per-correct via production theory, our primary metric.
\paragraph{Finance and real options.} Mean--variance allocation \citep{markowitz1952}, cost-aware diversification
\citep{statman1987}, the value of information \citep{howard1966}, selective prediction \citep{geifman2017}, switching
costs \citep{klemperer1995}, and real-options volatility comparative statics \citep{mcdonald1986,dixit1994,merton1976}
are the antecedents we specialize and, where standard, concede.

\section{Problem Formulation}\label{sec:setup-formal}
Queries $x$ carry a latent type $t=T(x)\sim D$, with type prior $p(t)=\Pr[T=t]$. A pool $M=\{1,\dots,m\}$ of models has
quality $q_i(t)\in[0,1]$ (expected per-query utility on type $t$) and price $c_i\ge 0$ (dollars/query); write
$\bar q_i=\E_t q_i(t)$. A (possibly stochastic) routing policy $\pi\colon T\to\Delta(M)$ has value
$V(\pi)=\E_t\sum_i \pi(i\mid t)\,q_i(t)$ and cost $K(\pi)=\E_t\sum_i \pi(i\mid t)\,c_i$. We also consider fusion (query
several, combine) and cascading (escalate on low confidence). The buyer's objective is dollars-per-correct
\citep{erol2025costofpass}, or quality subject to a budget; we make it explicit in each section.

\paragraph{Economic scaffolding (deferred to App.~\ref{app:econ}).} Treating orchestration as allocation yields a
compact economic layer that we use but do not foreground; the tools are standard and the empirical results below do
not depend on it. Three facts, stated and proved in App.~\ref{app:econ}: under a dollar budget, routing is a priced
assignment with a single shadow price $\lambda_B$ (Prop.~\ref{prop:duality}); cost-aware fusion has a diversification
limit $k^\ast(\rho,c)$ that shrinks as error correlation rises (Props.~\ref{prop:keven},~\ref{prop:lambda}), above the
classical equicorrelated variance floor we concede (Prop.~\ref{prop:floor}); and a calibrated cascade collapses to
random mixing exactly as the verifier AUC falls to $1/2$, with a price-independent escalation ceiling $1-a_L/a_H$
(Prop.~\ref{prop:cascade}, Cor.~\ref{cor:edge}).

\section{Experimental Setup}\label{sec:setup}
We executed a pre-registered program on a pool of 15 current models across 9 provider families: frontier (Claude
Opus~4.8, GPT-5.1, Gemini~3.1~Pro, Kimi~K2.7), mid (Claude Sonnet~4.6, GPT-5-mini, Gemini~3.5~Flash, Qwen3-235B, Mistral-Large,
MiniMax~M2.7, DeepSeek~V3.2), and cheap (Claude Haiku~4.5, GPT-5-nano, Gemini~3.1~Flash-Lite, Llama-4-Maverick); exact dated
snapshots and prices are frozen in the registry (App.~\ref{app:repro}). The pillar experiments use five benchmarks: a saturated
mix (GSM8K, MMLU, ARC-Challenge, MATH-500) and a harder set (MMLU-Pro), with $100$--$200$ queries per dataset. For the
\emph{market-scale} realizability measurement (\S\ref{sec:realiz}) we expand the pool to $67$ models across $21$ provider
families---the live OpenRouter catalog from the current frontier down to small open-weights (GLM, Qwen, DeepSeek, MiniMax,
Nemotron, Llama-3.x, Mistral, Gemma, Phi, Granite, and others), chat/instruct only, with live-verified prices (full named roster, App.~\ref{app:roster})---and add hard domains that probe co-failure: open-ended competition
mathematics on two benchmarks (MATH-500 and the harder MATH-Hard Level-5; plus AIME-2024/2025, whose release post-dates the older
models' training cutoff) and graduate-level science (GPQA-Diamond, Physics/Chemistry/Biology). Grading is programmatic throughout: exact-match arithmetic, multiple-choice and
boxed-letter extraction, and boxed/integer answer matching, so no LLM judge is used.
Costs are metered per call against the OpenRouter account usage endpoint; OpenRouter is an aggregator, so this is
account-level usage, not per-provider reconciliation. We itemize per-run metered cost in App.~\ref{app:repro}: the core pillar
experiments total $\approx$\$47, the market-scale realizability and two-regime measurement $\approx$\$111, and the two
third-domain experiments (code, open-ended GPQA) $\approx$\$110, for $\approx$\$270 of reported-experiment cost; total account
usage including all exploratory and superseded iteration was $\approx$\$560 (approximate; see App.~\ref{app:repro}). We report
the itemized experiment figures rather than presenting account-level usage as experiment cost. Baselines: single-cheapest,
single-best (selected in-sample; the optimistic bias \emph{understates} the oracle gain $G=V^o-a_{\mathrm{sb}}$, so our small-$G$ claim is conservative, though it flatters the learned-router comparison), random and random-mixing-at-matched-budget, cost-matched
Self-MoA \citep{li2025selfmoa}, a partition-conditioned oracle and a partition-free per-query oracle, a confidence cascade,
majority vote, and heterogeneous fusion.

\section{Results}\label{sec:results}
All correctness is scored by an answer-anchored grader; an earlier first-letter extractor systematically
mis-scored verbose models (e.g.\ Llama-4-Maverick by $+0.26$ accuracy, mean $|\Delta|=0.05$ across the pool), so we
re-graded all cached model outputs at no additional inference cost and report the corrected numbers throughout.

\begin{table}[h]\centering\small
\begin{tabular}{lcc}
\toprule
Quantity & Saturated multi-domain mix & Hard single-domain (MMLU-Pro) \\
\midrule
Single-best & 0.923 (Opus 4.8) & 0.850 (Sonnet 4.6) \\
Oracle (per-query) & 0.967 & 0.970 \\
\textbf{Oracle gain $G$ (95\% CI)} & \textbf{0.044} $[0.027,0.062]$ & \textbf{0.120} $[0.075,0.155]$ \\
mean off-diagonal $\rho$ & 0.464 & 0.382 \\
within-family $\rho$ & \textbf{0.528} & \textbf{0.402} \\
cross-family $\rho$ & 0.459 & 0.380 \\
\bottomrule
\end{tabular}
\caption{Pillar A in two regimes (re-graded; $G$ with $2000$-resample query-bootstrap 95\% CIs, $N{=}120$--$200$). $G>0$ with
both CIs excluding zero confirms $Q$ is not row-dominated (Lem.~\ref{lem:env}), and $G$ is larger on the harder, more
dispersed regime. Within-family $\rho$ exceeds cross-family in both regimes, with a larger gap on the multi-domain mix ($0.069$
vs $0.022$): family specialization shows most across domains, consistent with the shared-provider correlation of
\citet{kim2025correlated}.}\label{tab:A}
\end{table}

\paragraph{Pillar A (confirmed, both regimes).} Oracle gain $G>0$ with bootstrap CIs excluding zero in both regimes
($0.044$ saturated, $0.120$ hard; Table~\ref{tab:A}): routing helps, \emph{modestly because the frontier agrees}, and
more on the harder, more dispersed regime, the dispersion signature the theory predicts. Within$>$cross family $\rho$
holds in both regimes (gap larger on the multi-domain mix). The cost--quality frontier is populated by cheap models
(Fig.~\ref{fig:cq}). A deployable learned router captures essentially none of $G$, and this holds against router strength: a
held-out TF-IDF$+$domain logistic attains $0.906$ vs.\ single-best $0.901$ on the mix ($9\%$ of $G$, 95\% CI $[-0.67,0.50]$),
and, to rule out a weak-baseline artifact, three substantially stronger routers fare no better. A gradient-boosted per-model
correctness predictor on word$+$char TF-IDF features captures $-0.09$ of $G$; a direct multiclass best-model predictor
captures $-1.27$ (it actively hurts); and a deployment-realistic \emph{LLM-as-router} (GPT-5-mini shown each query and a capsule
of every model's strengths, asked to pick the best) routes to single-best on $100\%$ of queries and captures exactly $0$ of $G$
(\texttt{router\_strong.py}, \texttt{router\_llm.py}). All four routers are evaluated on the $15$-model saturated mix, the
pool whose per-query prompts we logged; the market-scale ($67$-model) and GPQA matrices store outcomes but not prompts, so no
router was trained there and the market-scale routing statement rests on the certificate of Prop.~\ref{prop:cert}, not an
end-to-end routing run. This scope limit we state plainly rather than paper over. Against the
cost-aware oracle (optimal) frontier all sit far below the upper bound (Fig.~\ref{fig:router}). The realizable routing gain is thus
near zero \emph{not because the router is weak} but because the prompt carries little signal about which model will be the one
that is right when the frontier disagrees: the small oracle bound is itself largely unreachable.

\paragraph{Tail co-failure: a realizability certificate and an empirical finding (\S\ref{sec:realiz}).}\label{sec:realiz}
Why is the realizable gain near zero, and the oracle gain itself small? Both are governed by how often the pool fails
together. The next proposition makes the ceiling exact and turns it into a \$0 pre-deployment test; we then report what the
tail looks like on the frontier.

\begin{proposition}[Ceiling, gain localization, and a realizability certificate]\label{prop:cert}
Let $\beta=\Pr_t[\text{all } m \text{ wrong}]$, $a_{\mathrm{sb}}=\max_i\bar q_i$, and $i^\star=\arg\max_i\bar q_i$.
\emph{(i)~Ceiling.} Any selection policy---a router, a (weighted) vote, or a cascade, whose output is almost surely one of
the members' answers---has accuracy at most $1-\beta$, attained by the per-query oracle, so the maximum gain over single-best
is exactly $\Delta^{\mathrm{ceil}}=(1-\beta)-a_{\mathrm{sb}}$. \emph{(ii)~Gain localization.} $G=V^o-a_{\mathrm{sb}}=\Pr_t[\text{single-best wrong}]-\beta$,
supported entirely on the \emph{resolvable} mass (non-unanimous, single-best wrong); the co-failure tail $\beta$ contributes
nothing to $G$. \emph{(iii)~Certificate.} From $n$ i.i.d.\ queries with $K$ all-wrong, let $\beta_{\mathrm{lo}}(K,n,\delta)$ be
the Clopper--Pearson lower confidence limit; then with probability $\ge1-\delta$ every selection policy obeys
$\mathrm{Acc}-a_{\mathrm{sb}}\le(1-\beta_{\mathrm{lo}})-a_{\mathrm{sb}}$. If this certified bound falls below the orchestration
overhead, no policy in the class can pay for itself---a \$0 test ($a_{\mathrm{sb}}$ replaceable by its own confidence bound).
\end{proposition}
\noindent Parts~(i)--(ii) are elementary identities, with one-line proofs (App.~\ref{app:proofs}): on the all-wrong event every member is wrong, so any selector is wrong;
and $G$ rearranges $V^o=1-\beta$. We state them not as deep results but because
part~(iii) turns them into a pre-deployment \$0 instrument, and part~(ii) corrects a tempting
misstatement: a \emph{small} $\beta$ does not by itself imply orchestration cannot
help; it implies a high ceiling. The binding quantity is $\Delta^{\mathrm{ceil}}=(1-\beta)-a_{\mathrm{sb}}$, which is small on
the frontier only because $a_{\mathrm{sb}}$ already sits near the ceiling; the certificate is most informative exactly there.

\noindent\textbf{The empirical finding (measured, not a law).} The statistic the field reports, mean pairwise error
correlation $\rho$, systematically underprices this tail. Fitting a single-factor Gaussian copula to the measured $\rho$
predicts an all-wrong rate $\beta_{\mathrm{sf}}$ far below the observed $\beta$ (Table~\ref{tab:tail}), and the gap persists
when the pool is restricted to one model per provider family (not a same-vendor artifact). The implied tail correlation that
reproduces $\beta$ far exceeds the pairwise value---the body-vs-tail base-correlation smile of Gaussian-copula
portfolio-credit (CDO) models, which we invoke as a known analogy, not a new object. Because $\beta$ rests on few all-wrong
events, we report it with exact Clopper--Pearson intervals; the underpricing factor is the headline, but its magnitude
carries real uncertainty, which the market-scale measurement below narrows.

\begin{table}[h]\centering\small\setlength{\tabcolsep}{3.5pt}
\begin{tabular}{lcc}
\toprule
 & Saturated mix & Hard (MMLU-Pro) \\
\midrule
all-models-wrong rate $\beta$ (95\% CP) & 0.033 [0.019, 0.054] & 0.030 [0.011, 0.064] \\
mean pairwise $\rho$ (\emph{naive Pearson-of-indicators}) & 0.464 & 0.382 \\
$\beta$ predicted by single-factor copula at that $\rho$ & 0.0011 & 0.0050 \\
\emph{naive-Pearson} underpricing (\textbf{overstated}; cf.\ tetrachoric below) & $30\times$ [17, 48] & $6\times$ [2, 13] \\
implied (tail) correlation reproducing $\beta$ & 0.88 & 0.64 \\
realizable router gain (fraction of $G$) & 0.09 (CI spans 0) & $<0$ \\
\bottomrule
\end{tabular}
\caption{Tail co-failure on the 15-model frontier pool (recomputed from the logged re-graded matrices via
\texttt{realizability.py}, \$0; $\beta$ with exact Clopper--Pearson 95\% intervals, $n{=}480$/$200$, all-wrong counts
$k{=}16$/$6$). The $6$--$30\times$ figures use the \emph{naive Pearson-of-indicators} calibration and are \textbf{overstated}---we
retain them only to show the raw gap; the correctly \emph{tetrachoric}-calibrated residual is single-digit ($\approx\!2.5\times$ on
the market-scale MATH-500 tail, Fig.~\ref{fig:realizability}), the order-of-magnitude difference being the calibration artifact we
diagnose in \S\ref{sec:results}. The wide intervals reflect the few all-wrong events, which the market-scale measurement narrows.
The ceiling $1-\beta$ and the near-zero realizable router gain both follow from this tail, not from $\rho$ (Prop.~\ref{prop:cert}).}\label{tab:tail}
\end{table}

\paragraph{The mispricing is a large-pool phenomenon (market scale).} We expand to a $67$-model, $21$-family market pool, the
live OpenRouter frontier (GPT-5.5, Claude Opus 4.8, Gemini 3.1 Pro, Grok-4.3, GLM-5.2, Qwen3.7-Max, DeepSeek V4, Kimi K2.7,
MiniMax M3) down to small open-weights, over the hard benchmarks (Fig.~\ref{fig:realizability}). The load-bearing benchmark is
\emph{MATH-500}, the one domain with enough co-failure events to estimate $\beta$ at all: over the full $67$-model pool ($n{=}330$ fully-covered queries)
all models miss the same problem $\beta=0.052$ of the time, but this rests on only $k{=}17$ all-wrong events (Clopper--Pearson $[0.030,0.081]$, a wide interval). The
single-factor copula must be calibrated with the \emph{tetrachoric} (latent) correlation, not the Pearson correlation of $0/1$
correctness indicators (Prop.~\ref{prop:lambda}), a century-old psychometric point \citep{pearson1900tetrachoric,olsson1979polychoric}
we claim no novelty for, but which the LLM-evaluation literature routinely elides. The measured tetrachoric $\bar\rho=0.78$ predicts $\beta_{\mathrm{sf}}=0.021$, so
the observed tail is $\approx\!\mathbf{2.5\times}$ fatter (bootstrap $90\%$ CI $1.7$--$3.4\times$ over queries, jointly propagating the
all-wrong count and the fitted $\bar\rho$; \texttt{ratio\_uncertainty.json}) than even a correctly-calibrated single-factor model: a real but
\emph{modest} residual common-mode excess (the implied $\rho_{\mathrm{eff}}=0.89$ exceeds the measured $0.78$). \textbf{The residual is
not an artifact of the single-factor restriction.} Fitting the \emph{full} $67{\times}67$ pairwise-tetrachoric correlation matrix
$\Sigma$ (every pair calibrated to its own joint wrong-rate, projected to the nearest PSD matrix) and Monte-Carlo-integrating the
all-wrong event under the resulting Gaussian copula still predicts only $\beta_{\textrm{full-}\Sigma}=0.023$ ($4{\times}10^5$ draws;
\texttt{residual\_decomp.py}), leaving the empirical $0.052$ a $\mathbf{2.25\times}$ excess beyond the \emph{nearest-PSD Gaussian
copula}---the finite-pool signature of a common-mode atom (Props.~\ref{prop:poolbias},~\ref{prop:nonid}), which a Gaussian copula
cannot represent (zero lower tail dependence), not single-factor misspecification. Two caveats keep this exact: the empirical
tetrachoric matrix is indefinite ($26$ negative eigenvalues) and the PSD projection \emph{lowers} the mean calibrated correlation
($0.78\!\to\!0.74$)---which, if anything, \emph{inflates} this ratio; and a non-Gaussian copula with \emph{genuine} lower-tail
dependence does no better on the real data. An exchangeable Clayton copula ($\lambda_L=0.69$), calibrated to the same mean
pairwise co-failure on the full $67$-model matrix, still predicts $\beta=0.026$ versus the empirical $0.052$, a $1.96\times$
residual ($6.3\times$ on MATH-Hard; \texttt{clayton\_real.py}). So the gap is \emph{not} an artifact of the Gaussian's zero tail
dependence. The residual therefore lies beyond any exchangeable pairwise-calibrated copula we can fit, Gaussian \emph{or}
tail-dependent: the signature of a common-mode atom that no pairwise statistic represents. We flag the
calibration trap explicitly, because an earlier version of this analysis fell into it: the \emph{naive} Pearson-of-indicators
calibration ($\bar\rho=0.53$) gives $\beta_{\mathrm{sf}}=0.0016$ and a spurious $32\times$: an order-of-magnitude artifact of the
wrong correlation transform, not a co-failure effect (\texttt{realizability\_tetrachoric.json}). \textbf{The excess is a pool-\emph{size}
effect, not a composition accident.} Resampling pool composition (random $k$-model subsets, $60$ per $k$) the tetrachoric ratio rises
monotonically from $1.0$ at $k{=}2$ to a median $2.5$ ($5$--$95\%$ band $[2.1,2.7]$) at $k{=}67$, with \emph{every} subset showing a
populated tail (\texttt{residual\_decomp.json}): isolating size, not which models, as the driver, exactly as Prop.~\ref{prop:poolbias}
predicts; the newest frontier (GPT-5.5 and peers) still co-fails.
\textbf{The finding replicates on a second, harder open-ended math benchmark}, though both are the \emph{same task family} at two
difficulties, not independent domains. On MATH-Hard (Level-5 MATH; $67$ models, $n{=}298$) the co-failure tail is again populated
($\beta=0.044$, $k{=}13$ all-wrong, CP$[0.023,0.073]$), and the tetrachoric-calibrated single-factor model underprices it by a point
$8.3\times$ ($\bar\rho_{\mathrm{tet}}=0.69$; bootstrap $90\%$ CI $4.5$--$16\times$). We deliberately do \emph{not} read this larger point
ratio as stronger co-failure. MATH-Hard's $\beta$ ($0.044$) is in fact \emph{lower} than MATH-500's ($0.052$); the higher ratio is a
\emph{denominator} effect: its lower fitted $\bar\rho_{\mathrm{tet}}$ shrinks the single-factor baseline $\beta_{\mathrm{sf}}$. Matched to
MATH-500's $\bar\rho=0.78$, MATH-Hard's ratio is $3.3\times$ (\texttt{ratio\_uncertainty.json}), comparable to MATH-500's. The honest
reading is therefore a \emph{consistent single-digit} residual common-mode excess ($\approx\!2.5$--$3.3\times$ at matched $\rho$; point
ratios $2.5$--$8.3\times$, both with wide CIs over $k{=}17$ and $13$ events), replicated \emph{within} open-ended math but not across
domains. The thinner MMLU-Pro tail (one all-wrong event in $124$) we read only as directional, and on multiple-choice GPQA the tail
vanishes ($\beta\!\approx\!0$). \textbf{The co-failure regime generalizes to a third, structurally independent domain.} On
execution-graded competitive programming (\texttt{code\_contests}: $63$ hard problems, rating $1900$--$3500$, each graded against
its \emph{private} $+$ \emph{generated} stress tests under an enforced, Python-fair time limit; \S\ref{app:codegen}), the tail is
populated ($\beta=0.079$, $k{=}5$ all-wrong over $63$, CP$[0.026,0.176]$). The same naive-Pearson trap recurs (Pearson
$\bar\rho{=}0.27$ implies a spurious $17\times$), the tetrachoric single-factor model underprices by $\mathbf{3.1\times}$, and even
the full-$\Sigma$ Gaussian copula leaves a $1.7\times$ residual (\texttt{residual\_decomp.json}): the same common-mode signature as
math. With the harder problems the underpricing is now \emph{statistically resolved}: bootstrap $90\%$ CI $[1.5,6.2]$, excluding
$1$. The honest caveats remain: $k{=}5$ is still a small event base, $18$ models (not $67$), and a strict-but-not-official judge
(App.~\ref{app:codegen}). The signature thus holds across \emph{three} structurally disjoint open-ended domains (two math
families and execution-graded code) and \emph{vanishes} on multiple-choice: co-failure ($\beta>0$), the Pearson trap, a
full-$\Sigma$ residual, and a correlation-excluding-$1$ tetrachoric ratio. The open-ended-versus-multiple-choice split is a
cross-domain phenomenon, not a math artifact.

\begin{figure}[tbp]\centering
\includegraphics[width=0.92\textwidth]{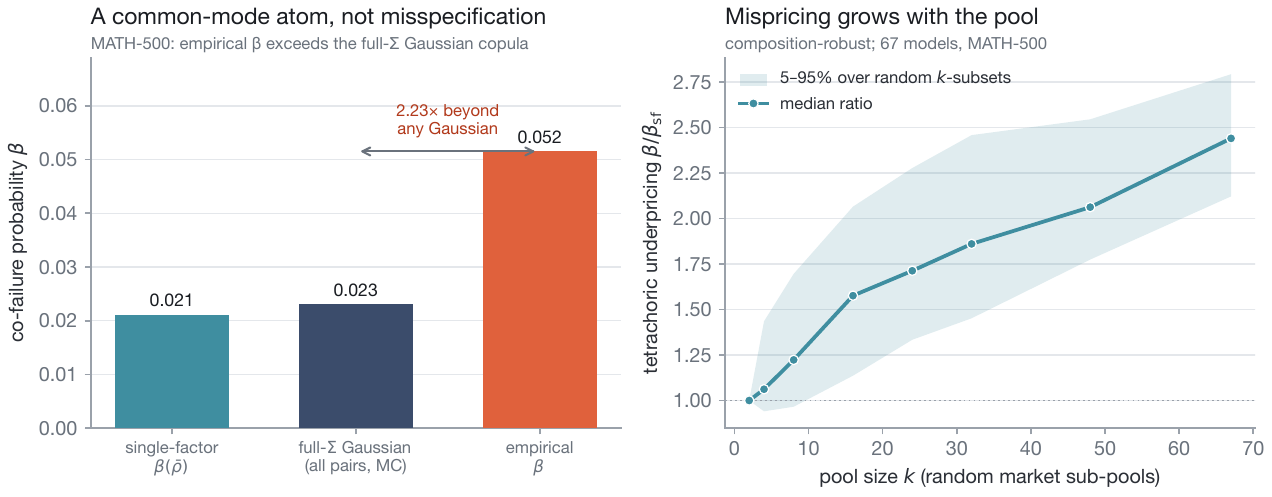}
\caption{\label{fig:realizability}The co-failure residual is a common-mode atom, not copula misspecification (MATH-500, $67$
models, $k{=}17$/$330$). \emph{Left}: three predictions of the all-models-wrong rate against the empirical $\beta=0.052$---the
one-parameter single-factor copula ($\beta(\bar\rho)=0.021$), and the \emph{full} $67{\times}67$ pairwise-tetrachoric Gaussian
copula Monte-Carlo'd over all pairs ($\beta_{\textrm{full-}\Sigma}=0.023$). The empirical tail exceeds even the nearest-PSD
full-$\Sigma$ Gaussian fit by $2.25\times$, whose lower tail is asymptotically independent (Props.~\ref{prop:poolbias},~\ref{prop:nonid}). (The discredited Pearson-of-indicators calibration would put
$\beta_{\mathrm{sf}}=0.0016$, a spurious $32\times$; we exclude it.) \emph{Right}: resampling pool \emph{composition} (random
$k$-model subsets, $60$ per $k$), the tetrachoric underpricing rises monotonically from $1.0$ at $k{=}2$ to median $2.5$
($5$--$95\%$ band $[2.1,2.7]$) at $k{=}67$---size, not which models, drives it. Computed by \texttt{residual\_decomp.py}.}
\end{figure}

The monotone growth is not an artifact of the scan; it is forced by any positive tail dependence, which we now prove.

\begin{proposition}[Pairwise $\rho$ underprices co-failure, with bias growing in pool size]\label{prop:poolbias}
Model errors by a common-shock mixture: each query is \emph{co-hard} with probability $\pi$ (all $m$ models err together) and
otherwise each model errs independently with probability $\alpha_0$. The marginal error rate is $\alpha=\pi+(1-\pi)\alpha_0$ and
the pairwise error correlation is $\bar\rho=[\pi+(1-\pi)\alpha_0^2-\alpha^2]/[\alpha(1-\alpha)]>0$. The true co-failure rate is
$\beta(m)=\pi+(1-\pi)\alpha_0^{\,m}$. Let $\beta_{\mathrm{sf}}(m)$ be the all-wrong rate of the single-factor Gaussian copula
calibrated to $(\alpha,\bar\rho)$. Then (i) $\beta(m)\downarrow\pi>0$ while $\beta_{\mathrm{sf}}(m)\downarrow0$, because a Gaussian
copula with $\bar\rho<1$ has zero lower tail dependence \citep{sibuya1959bivariate,embrechts2002correlation}; hence (ii) the underpricing ratio
$\beta(m)/\beta_{\mathrm{sf}}(m)\to\infty$ and is eventually strictly increasing in $m$, while equalling $1$ at $m=2$. Mean
pairwise $\rho$ is a sufficient statistic for the \emph{bivariate} error law but discards the higher-order tail dependence that
governs joint failure of a large pool; it is exact for pairs and increasingly inadequate as the pool grows.
\end{proposition}
\noindent The classical ingredient---a Gaussian/elliptical copula has zero lower tail dependence
\citep{sibuya1959bivariate}, so a common-mode atom (Marshall--Olkin-type shared-failure component;
\citealp{marshall1967multivariate}) cannot be represented by any pairwise calibration---is not ours; the body--vs--tail
underpricing it implies is the base-correlation ``smile'' familiar from Gaussian-copula credit models
\citep{li2000default,donnelly2010devil}. What we own is the \emph{transfer} to LLM orchestration: the co-failure
instantiation, the pool-size-divergence framing, and the empirical measurement that it does not vanish on the real frontier.
It grounds the paper's flagship empirical claim: the measured curve (ratio $\approx1$ at $k{=}2$, growing monotonically with
$k$) has the monotone-divergent shape this predicts. The exchangeable, homogeneous model licenses the mechanism and the sign---any
positive tail dependence forces the divergence---while the heterogeneous magnitude (the specific $1.3\times\!\to\!2.5\times$ tetrachoric on
MATH-500) is empirical. It also says precisely \emph{why} a buyer cannot price orchestration from the reported statistic:
$\rho$ certifies pairwise substitutability, not the tail in which orchestration either pays or does not.

\noindent\textbf{The driver is a common mode, not tail dependence \emph{per se}.} It is tempting to attribute the effect to
the lower tail-dependence coefficient $\lambda_L$, but that is false, and the distinction is the substance of the result. What
pairwise $\rho$ misses is a \emph{common-mode atom}: a positive-probability event of joint failure, $\beta_\infty:=\lim_m\beta(m)>0$,
that the pairwise correlation underweights. Smooth tail dependence that is \emph{already reflected} in $\rho$ does not produce the
effect. We verified the dichotomy in a logged simulation (\texttt{copula\_dichotomy.py}): an exchangeable Clayton copula
(Archimedean, $\lambda_L=2^{-1/\theta}$) yields only \emph{bounded} underpricing under the same binary-$\rho$ pipeline---$4.0\times$
at $m{=}53$ for $\lambda_L{=}0.71$, and \emph{smaller} ($2.6\times$) for larger $\lambda_L{=}0.84$, because higher $\lambda_L$
raises the pairwise $\rho$ the single-factor model then absorbs---whereas a common-shock mixture with the same marginals but a
rare shared-failure atom ($\beta_\infty{=}0.05$) drives the ratio past $10^7\times$. The
orchestration-relevant object is thus the multivariate co-failure floor $\beta_\infty$, which no single pairwise number can
identify; this both sharpens the certificate of Prop.~\ref{prop:cert} and explains why the gap \emph{widens} with pool size
(pairwise $\rho$ saturates while $\beta_\infty$ does not). We state the underlying non-identification exactly.

\begin{proposition}[Non-identification of the co-failure floor; specialization of a classical Fr\'echet-class fact]\label{prop:nonid}
For $m\ge3$ the co-failure rate $\beta=\Pr[\text{all }m\text{ wrong}]$ is \emph{not} a function of the pairwise error law:
there exist joint distributions over $\{0,1\}^m$ with identical one- and two-dimensional marginals---hence identical marginal
error rates and identical pairwise correlations, Pearson \emph{and} tetrachoric---yet different $\beta$. Consequently no
statistic computed from pairwise correlations, a single-factor copula included, can identify $\beta$ or the large-pool floor
$\beta_\infty$; the common-shock atom and a matched-$\bar\rho$ Gaussian dependence (Prop.~\ref{prop:poolbias}) are precisely the
two extremes of this ambiguity ($\beta_\infty>0$ versus $\beta_\infty=0$ at identical pairwise law).
\end{proposition}
\noindent We claim no novelty for the mathematics: the proof (App.~\ref{app:proofs}) is the classical fact that low-order
marginals underdetermine a joint on $\{0,1\}^m$ for $m\ge3$ (the Fr\'echet class has positive dimension;
\citealp{fontana2018representation}), instantiated for the co-failure cell. The consequence for practice is that pairwise $\rho$ cannot, even in principle, see the
quantity that caps orchestration, so a direct estimate of $\beta$ (with the certificate of Prop.~\ref{prop:cert}), not any function of $\rho$, is
the right instrument.

\paragraph{Two regimes across three domains: either the ceiling binds, or it does not.} The gain-localization identity
$G=\Pr[\text{single-best wrong}]-\beta$ says orchestration headroom can be foreclosed in two opposite ways, realized across our
hard benchmarks (Table~\ref{tab:regimes}). On open-ended mathematics (MATH-500) the \textbf{ceiling binds}: a
real co-failure tail ($\beta=0.052$, $k{=}17$) caps every policy at $1-\beta$, which a correctly tetrachoric-calibrated single-factor model still underprices $\sim\!2.5\times$, and little
gain is even available. On graduate-level science (GPQA-Diamond, $52$-model complete-coverage subset) the picture \textbf{inverts}: all-models-wrong is
indistinguishable from zero ($0$ all-wrong on $130$ covered queries; $95\%$ Clopper--Pearson upper bound $0.03$), so the ceiling
is effectively open---and yet the oracle gain is \emph{larger}, $G=0.15$, of which the identity
certifies that essentially every point is resolvable disagreement, not co-failure. We did not train a router here---the
market/GPQA matrices log no prompts---so whether one could capture this gain is open; what the certificate shows is that the
gap is bounded by \emph{routing regret}, not by the tail, the opposite of the math regime. Execution-graded code
(App.~\ref{app:codegen}) sits in the \emph{ceiling-bound} regime alongside math ($\beta=0.079$), so the pattern holds across
three structurally disjoint open-ended domains: co-failure tracks \emph{open-endedness}. A multiple-choice guess floor over a
broad pool tends to make joint failure rare, though not impossible (MMLU-Pro shows a lone all-wrong event). So the operative
question is not ``how correlated are the models'' but ``which regime is this workload in,'' which pairwise $\rho$ cannot answer
and the certificate can (Fig.~\ref{fig:regimemap}). \textbf{The regime is set by format, not content---a content-controlled test.} We re-ran the
\emph{same} GPQA-Diamond questions as free-response (options stripped), graded by a $5$-judge LLM panel (pairwise
$\kappa=0.73$--$0.92$, substantial-to-near-perfect agreement; App.~\ref{app:opengpqa}). Holding the science content fixed and
changing only the \emph{format} \textbf{flips the regime}: on the models and questions common to both runs, mean accuracy falls
$0.66\!\to\!0.51$ (best-model $0.91\!\to\!0.77$), and a co-failure tail \emph{opens}---$\beta=0.127$ ($k{=}10$ over $79$
fully-judged questions, CP$[0.062,0.220]$), comparable to math and code, where the multiple-choice version had $\beta\approx0$. So co-failure tracks \emph{open-endedness} itself, not subject matter, and it
appears even on an LLM-judged generative task, not only programmatically-verifiable ones---the strongest evidence that the
open-ended/multiple-choice split is real and not a domain confound (Fig.~\ref{fig:flip}).

\begin{figure}[t]\centering
\includegraphics[width=0.94\textwidth]{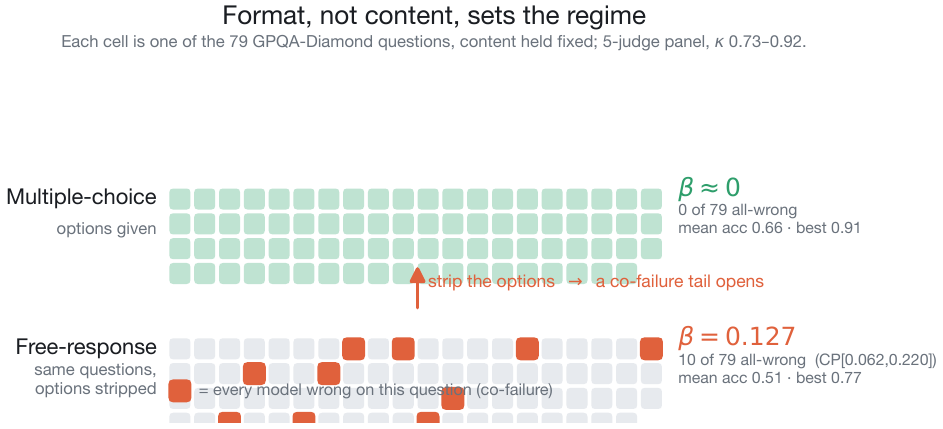}
\caption{\label{fig:flip}\textbf{Format, not content, sets the regime.} The same $79$ GPQA-Diamond questions, asked
multiple-choice (top) and free-response (bottom; options stripped, $5$-judge panel, $\kappa\,0.73$--$0.92$). Each cell is one
question; an orange cell is one on which \emph{every} model is wrong. Changing only the format opens a co-failure block of
$10/79$ ($\beta{=}0.127$, CP$[0.062,0.220]$) where multiple-choice had none ($\beta\approx0$), while mean accuracy falls
$0.66\!\to\!0.51$. Recomputed from the committed open-ended outcome matrix (\texttt{matrix\_marketGPQAOPEN}).}
\end{figure}

\begin{table}[h]\centering\small\setlength{\tabcolsep}{5pt}
\begin{tabular}{lccc}
\toprule
 & MATH-500 (math) & code\_contests (code) & GPQA-Diamond (science) \\
\midrule
models / queries & 67 / 330 & 18 / 63 & 52 / 130 \\
single-best & 0.836 & 0.825 & 0.846 \\
per-query oracle & 0.948 & 0.921 & 1.000 \\
oracle gain $G$ & 0.112 & 0.096 & \textbf{0.154} \\
all-models-wrong $\beta$ & \textbf{0.052} ($k{=}17$) & \textbf{0.079} ($k{=}5$) & $<0.03$ ($0/130$) \\
mean pairwise $\rho$ & 0.53 & 0.27 & 0.25 \\
$\rho$ underprices $\beta$ by (tetrachoric) & $\mathbf{2.5\times}$ & $\mathbf{3.1\times}$ (CI $[1.5,6.2]$) & --- ($\beta\!\approx\!0$) \\
\textbf{regime} & \textbf{ceiling-bound} & \textbf{ceiling-bound} & \textbf{realizability-bound} \\
\bottomrule
\end{tabular}
\caption{Three structurally disjoint open-ended domains, two opposite ways orchestration headroom is foreclosed (2026 frontier
market pool: MATH-500 on the full $67$ models; execution-graded code on $18$ models / $63$ reference-certified problems
(App.~\ref{app:codegen}; $\beta$ CP $[0.026,0.176]$, underpricing CI $[1.5,6.2]$ now excludes $1$, though $k{=}5$ is still small);
GPQA on the $52$-model complete-coverage subset---the newest reasoning models' GPQA cells were incomplete at
the budget cap, $\beta$ there indistinguishable from $0$). Where the ceiling
binds (math), pairwise $\rho$ underprices the binding co-failure tail; where it is slack (science), a large oracle gain is pure
resolvable disagreement a deployable router still misses. Pairwise $\rho$ identifies neither regime.}\label{tab:regimes}
\end{table}

\begin{figure}[t]\centering
\includegraphics[width=0.97\textwidth]{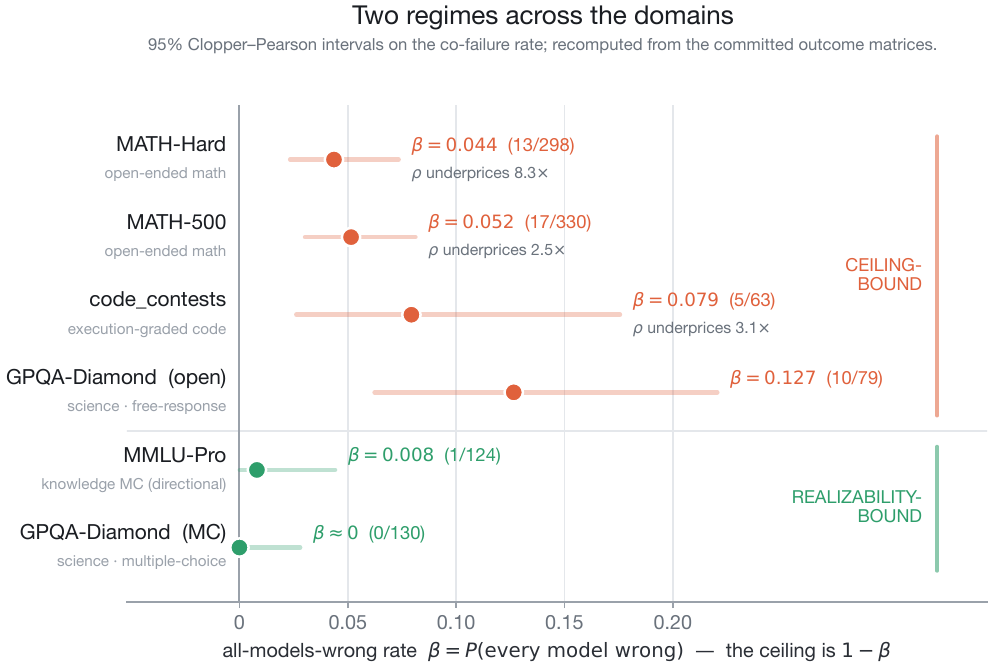}
\caption{\label{fig:regimemap}\textbf{Two regimes across the domains.} All-models-wrong rate $\beta$ per domain with $95\%$
Clopper--Pearson intervals, recomputed from the committed outcome matrices. \emph{Ceiling-bound} domains (open-ended math and
code, free-response GPQA) carry a co-failure tail $\beta>0$ that caps every selection policy at $1-\beta$ and that pairwise
$\rho$ underprices by $2.5$--$8.3\times$ (tetrachoric); \emph{realizability-bound} domains (multiple-choice GPQA, MMLU-Pro) have
$\beta\approx0$, so the oracle gain is pure resolvable disagreement a router could in principle capture. The same GPQA-Diamond
content sits in \emph{both} regimes depending only on answer format (Fig.~\ref{fig:flip}).}
\end{figure}

\paragraph{Pillar B (two complementary findings).} \emph{(1) Naive diversity is a liability.} On all $\binom{15}{3}=455$
three-model triplets, regressing the unweighted majority-vote gain over the best member on $\rho$ with accuracy-headroom control and
model-clustered (leave-one-model-out jackknife) inference, the mean vote gain is negative ($-0.10$ hard, $-0.02$
saturated, robust across the jackknife): mixing unequal-quality models lets diverse-but-weaker members outvote a strong one. The
finer prediction that the gain rises with $1-\rho$ has a positive point estimate (slope $+0.13$ hard) but is not significant
under model-clustering (jackknife 95\% CI $[-0.07,+0.34]$; the naive i.i.d.\ bootstrap CI is $\approx\!3\times$ too narrow), so we
report it as suggestive (Fig.~\ref{fig:rho}). This refutes the ``more diverse $\Rightarrow$ better fusion'' intuition and matches
\citet{li2025selfmoa,kuncheva2003}. \emph{(2) At matched quality, the diversification mechanism is supported in one regime.} We test the diversification
theorem in its valid regime by contrasting Self-MoA (distinct samples of the single best model; high intra-model error correlation
$\rho{=}0.80$) against heterogeneous fusion over an accuracy-matched 6-model band (members $0.74$--$0.865$; lower inter-model
correlation $\rho{=}0.42$, estimated on a disjoint sample split), with $S{=}9$ so Self-MoA scales on distinct draws rather
than recycling a fixed budget (Fig.~\ref{fig:eqq}). At the \emph{information-fair} comparison ($k{=}3$, equal distinct draws per
side), the low-correlation heterogeneous ensemble beats Self-MoA at $k{=}3$. Across $60$ resamplings of the sample-to-split partition
the gain averages $+0.027$ (range $+0.010$ to $+0.050$; positive in all $60$), so the \emph{direction} is robust while the
magnitude is partition-sensitive. The $+0.055$ a single favourable partition gives (query-bootstrap CI $[+0.025,+0.090]$) is the
upper end of this spread, and an alternative aggregation gives $+0.025$ near the mean; we therefore report the partition-averaged
$+0.027$ and treat the mechanism as supported in one regime, not established (\S\ref{sec:limits}). This is consistent with the
diversification limit's core prediction: lower inter-model error correlation buys larger diversifiable gains at matched quality. The advantage also shrinks as $\rho$ rises, the direction
$k^\ast(\rho)$ predicts. On the higher-correlation MATH-500 regime ($\rho_{\mathrm{inter}}{=}0.59$) it falls to $+0.020$ (not
significant). Across up to 502 matched-quality sub-bands of a 9-model band, the accuracy-controlled gain-vs-$\rho$ slope is
robustly negative (the diversification-limit sign, reversing the unmatched pool's positive slope) and stable across 6- and
9-model pools. It is not significant under model-clustering even with 9 clusters: the sign is real but small against
model-level variance (Fig.~\ref{fig:kstar}). We report the contrast in the Self-MoA frame, where $\rho_{\mathrm{intra}}$ and
$\rho_{\mathrm{inter}}$ are distinct estimands. The sensitivity $\lambda$ is now pinned down as a decision-rule Jacobian
(Prop.~\ref{prop:lambda}); fitting it across $\rho$ levels on data, to predict $k^\ast$ out of sample, remains future work.

\paragraph{Pillar C (confirmed, with caveats).} With $L=$GPT-5-nano ($a_L=0.748$ via $5$-sample self-consistency), $H=$Opus~4.8
($a_H=0.921$), and a self-consistency verifier ($\mathrm{AUC}=0.899$; near-binary at $k_c{=}5$), the collapse identity
\eqref{eq:collapse} is confirmed directly (Fig.~\ref{fig:casc}): averaged over $20$ noise-injection seeds, the cascade's
advantage over random mixing falls monotonically toward zero ($0.121\!\to\!0.012$, seed std $\le0.005$) as the verifier
degrades, while the injected-noise verifier AUC falls $0.899\!\to\!0.510$ (non-monotone at low injection, as expected; we
therefore plot the advantage against AUC, not against injection level). The volume ceiling is $1-a_L/a_H=0.188$. At the unconstrained optimum the cascade collapses to the
$L$ corner (it merely matches $L$ on dollars-per-correct); its dominance over $H$-only holds in the quality-constrained band
$q\in(a_L,a_H]$. A control replacing $H$ with Mistral-Large removes dominance, but because it lowers both $a_H$ and $H$'s price it
confounds the tail-edge effect with price, so we report it as suggestive (cf.\ \citealp{jitkrittum2023cascade}). We replace the
in-sample optimism with a $5$-fold \textbf{held-out} evaluation: choosing the deferral threshold on the train folds, the cascade
still beats random-mixing-at-matched-budget by $+0.114$ accuracy on the held-out fold (cross-fold sd $0.010$), with held-out
confidence AUC $0.899$---the dominance is not an in-sample artifact (\texttt{cascade\_heldout.py}). The one remaining cascade gap
is the optimal two-model deferral upper bound of \citet{jitkrittum2023cascade}, which requires $H$'s confidence (unlogged in our
matrices); we flag it rather than claim it.

\paragraph{Optionality under churn (secondary).} A separate observational study on the 2024--2026 release timeline (frontier
cost per unit capability fell $\approx\!14\times$) is not load-bearing for the $\beta$/$\rho$ result and is deferred to
App.~\ref{app:churn}.

\begin{figure}[tbp]
\begin{minipage}{0.49\textwidth}\centering
\includegraphics[width=\textwidth]{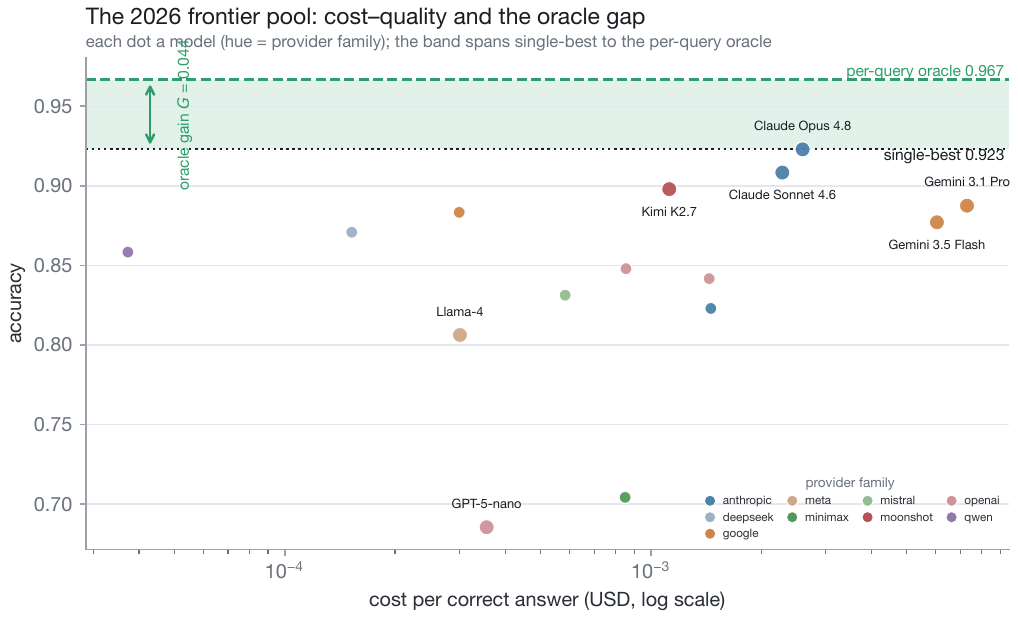}
\caption{\label{fig:cq}Pillar A cost--quality frontier (re-graded): the per-query oracle (green) sits just above single-best
(black); cheap models populate the frontier.}
\end{minipage}\hfill
\begin{minipage}{0.49\textwidth}\centering
\includegraphics[width=\textwidth]{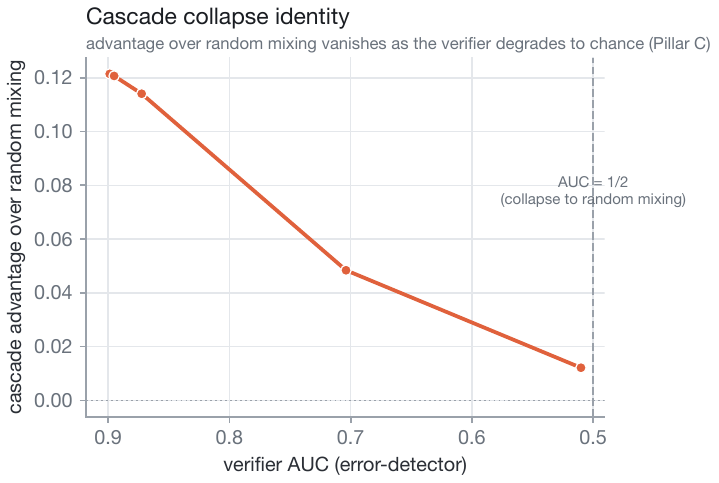}
\caption{\label{fig:casc}Pillar C: the cascade's advantage over random-mixing-at-matched-budget collapses to zero as the
verifier's AUC falls to $\tfrac12$, confirming Eq.~\eqref{eq:collapse}.}
\end{minipage}
\end{figure}
\begin{figure}[tbp]
\begin{minipage}{0.49\textwidth}\centering
\includegraphics[width=\textwidth]{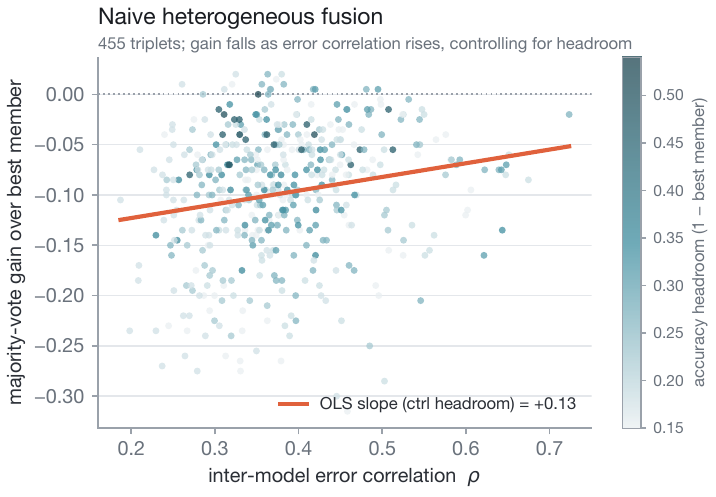}
\caption{\label{fig:rho}Pillar B, 455 triplets: unweighted majority-vote gain over the best member vs.\ $\rho$; gains are mostly
negative (robust), the positive slope not significant under model-clustering.}
\end{minipage}\hfill
\begin{minipage}{0.49\textwidth}\centering
\includegraphics[width=\textwidth]{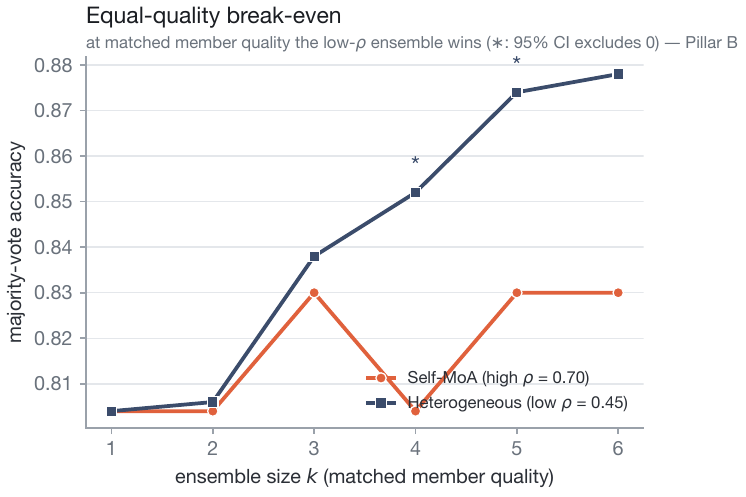}
\caption{\label{fig:eqq}Equal-quality break-even ($S{=}9$, matched 6-model band): low-$\rho$ heterogeneous fusion ($\rho{=}0.42$)
vs.\ high-$\rho$ Self-MoA ($\rho{=}0.80$), both on distinct draws so Self-MoA scales without recycling samples. Under the pre-registered
distinct-draw aggregation the heterogeneous ensemble beats Self-MoA from the information-fair $k{=}3$ onward ($\ast$:
query-bootstrap 95\% CI excludes zero), supporting the diversification mechanism at matched quality in this regime.}
\end{minipage}
\end{figure}
\begin{figure}[tbp]\centering
\includegraphics[width=0.55\textwidth]{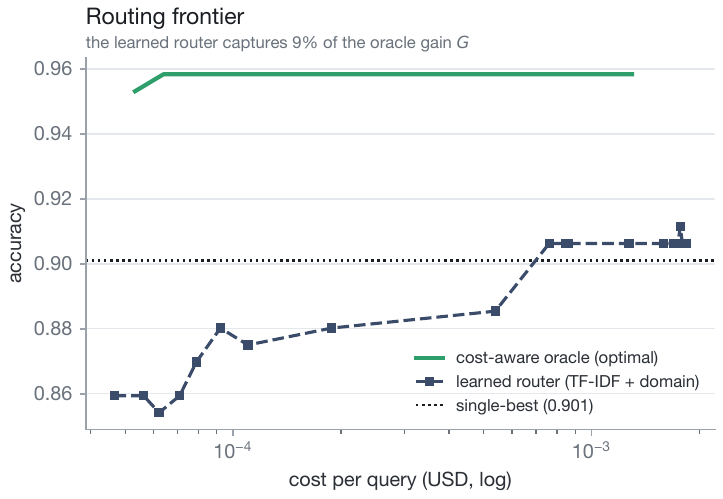}
\caption{\label{fig:router}Pillar A, realizable routing: a held-out learned router (TF-IDF$+$domain) vs.\ the cost-aware oracle
(optimal) frontier and single-best, on the multi-domain mix. The learned router barely exceeds single-best (captures $\sim$9\% of
$G$, CI spans zero) and lies well below the optimal frontier; realizable routing gain is near zero on the 2026 frontier.}
\end{figure}

\section{Discussion}\label{sec:disc}
\textbf{One allocation problem on two timescales.} The results are not separate vignettes but one allocation problem viewed
at two timescales. Within a release epoch, prices and the pool are fixed and the buyer solves the static allocation of
App.~\ref{sec:A}--\ref{sec:C}: a budget-priced assignment with value $V(B)$ and shadow price $\lambda_B$ (Prop.~\ref{prop:duality}),
capped by the realizability ceiling (Prop.~\ref{prop:cert}). Across epochs, frontier releases arrive and the buyer holds an
option on the next pool; the option value of breadth (App.~\ref{app:churn}) is the continuation value attached to that stage
problem. We do not claim a single solved Bellman system---the churn primitives $(\nu,\Gamma,v)$ are not derived from the stage
problem's---so this is a decomposition that organizes the results, explaining why the static claims hold within an epoch and the
option value across epochs; the continuation algebra itself is standard renewal/real-options machinery, which we concede. With that:
routing value is a first-moment selection effect that scales with dispersion rather than capability (App.~\ref{sec:A}); fusion value is a
second-moment effect bounded by systematic error and, empirically, realized only under accuracy-matched combination (App.~\ref{sec:B});
cascade value is a decision-theoretic effect equal to the integrated AUC lift of the verifier (App.~\ref{sec:C}). Because all three
shrink as the frontier converges and errors correlate (Cor.~\ref{cor:converge}), the value of a routing layer tracks market churn and
heterogeneity, not the absolute capability of the best model (App.~\ref{app:churn}). The empirical signature is already visible: on the 2026
frontier, oracle gains are small and \emph{naive} fusion is a net liability precisely because today's best models agree---yet once
member quality is matched, lower error correlation still buys a significant gain, so the lever is failure-mode heterogeneity, not count.

\section{Limitations}\label{sec:limits}
Programmatic grading covers verifiable tasks only and is sensitive to answer-extraction heuristics that can mildly penalize verbose
models; open-ended quality would reintroduce judge bias. Saturated benchmarks inflate $\rho$, mitigated but not removed by the hard
regime. The static-price assumptions of App.~\ref{sec:A}--\ref{sec:C} are in tension with the churn of App.~\ref{app:churn}; those claims are
restricted to within-release-epoch validity. The equal-quality assumption is empirically load-bearing (naive heterogeneous voting
hurts); the diversification mechanism is supported at matched quality (\S\ref{sec:results}). Its sensitivity $\lambda$ is now
derived as a decision-rule Jacobian (Prop.~\ref{prop:lambda}), but the empirical fit of $\lambda$ across $\rho$ levels and
out-of-sample prediction of $k^\ast$ remain open, and the matched-quality test rests on one provider-matched band; an
alternative aggregation pipeline gives a smaller, non-significant gain. Inference on the unconditional $\rho$-slope is inconclusive under model-clustering;
$G$ and the block-$\rho$ gap are reported without seed replication. The churn study (App.~\ref{app:churn}) is stylized and
observational. We instantiate a deployable learned router and the cost-aware oracle (optimal routing) frontier from logged outcomes
(\S\ref{sec:results}, Fig.~\ref{fig:router}): the router captures $\sim\!0$ of $G$, answering the routing question in the
negative on this pool. What remains is the optimal cascade-routing policy of \citet{dekoninck2024cascaderouting} as a
cascade-side upper bound (the cascade result is still measured against a naive confidence cascade, not the cascade optimum). Relatedly, our
cascade verifier scores only the cheap model ($L$), whereas the optimal deferral rule conditions on
both models \citep{jitkrittum2023cascade}, so our single-model AUC verifier is a practical, provably-dominated choice. The cascade
threshold is now validated out-of-sample by a $5$-fold held-out evaluation (App.~\ref{sec:C}); what remains on the cascade side is
only the optimal two-model deferral upper bound above. \textbf{External validity} is now supported across three structurally
disjoint open-ended domains---two math families and execution-graded code, where the co-failure signature (populated $\beta$, the
Pearson trap, a full-$\Sigma$ residual, and a tetrachoric underpricing CI that excludes $1$) replicates and inverts on
multiple-choice (App.~\ref{app:codegen})---though the code \emph{magnitude} still rests on $k{=}5$ events under a
strict-but-not-official judge on an $18$-model pool, so the point ratio carries real uncertainty. A precise code magnitude
(official hidden-test judge at scale on the full pool, with logged prompts for an in-domain router) is the one remaining
sharpening. The
remaining items for a top-venue submission are therefore: a tight-ratio code replication; $\ge 3$ seeds and held-out model selection on the matched-quality result, and its
extension across multiple $\rho$ levels to fit $\lambda$ and predict $k^\ast$ out of sample; the optimal cascade-routing
baseline above; and a price-controlled tail-edge experiment. (The option value of breadth is
already measured, not stylized, on a real generational timeline, \S\ref{sec:results}.) The authors' own orchestration workflow ran tiered but
single-provider with unmetered cost/latency and is reported only as a transparency note, never as evidence.

\section{Conclusion}
The field decides whether to orchestrate by reading one number, pairwise error correlation, and that number is blind to the
joint failures that set the ceiling. Treating orchestration as allocation over a correlated, priced, churning pool replaces it
with the right object ($\beta$, the co-failure tail), a \$0 certificate on the achievable gain, and an economics of when the gain
is reachable: a calibrated diversification limit, a cascade calibration boundary, and an option value of breadth that loads on
churn. Empirically, headroom is foreclosed two ways: a binding co-failure ceiling on open-ended tasks, where $\rho$ underprices
the tail, and a slack ceiling on others, where a large oracle gain is resolvable disagreement no deployable router yet captures.
Both say the same thing for practice---on our pool and verifiable tasks, and absent a strong query-level routing signal:
\emph{on open-ended tasks} the best models increasingly fail alike, so the lever is failure-mode dispersion and market churn, not
peak capability or model count. Pairwise correlation will not tell a buyer which lever they hold. Whether this holds on
open-ended generative tasks beyond our verifiable benchmarks is open.

\bibliographystyle{plainnat}
\bibliography{references}

\appendix
\section{Economic scaffolding: routing, diversification, and cascades}\label{app:econ}
This appendix gives the economic scaffolding deferred from the main text (\S\ref{sec:setup-formal}); its tools are standard, and the empirical claims do not rest on it.

\subsection{Routing as priced assignment}\label{sec:A}
\emph{The envelope and dispersion-scaling material below is background, claimed by no one}: the oracle envelope and the
value-of-information gap are textbook \citep{howard1966}, and routing optimality is proven by
\citet{dekoninck2024cascaderouting}. Proposition~\ref{prop:duality} makes ``orchestration
is allocation'' literal: under a dollar budget, routing is a priced assignment with an explicit shadow price.

\begin{lemma}[Envelope and dominance condition; background]\label{lem:env}
For every routing policy $\pi$, $V(\pi)\le V^o:=\E_t\max_i q_i(t)$, attained by the oracle $\pi^o(t)=\arg\max_i q_i(t)$.
Let the oracle gain be $G:=V^o-\max_i\bar q_i\ge 0$. Then $G>0$ iff no model is uniformly best across types
$D$-almost everywhere; equivalently, the quality matrix $Q=[q_i(t)]$ is not row-dominated under $D$.
\end{lemma}

\begin{proposition}[Budget-constrained routing: shadow price and bang-per-buck rule]\label{prop:duality}
Fix a dollar budget $B$ and let $V(B)=\max_{\pi}\{V(\pi):K(\pi)\le B\}$ over routing policies. This is a linear program
whose dual collapses to a \emph{single scalar} price $\lambda_B\ge0$ on the budget,
\[ V(B)=\min_{\lambda\ge0}\Big\{\lambda B+\E_t\max_i\big[q_i(t)-\lambda c_i\big]\Big\}, \]
because the per-type simplex constraints are absorbed pointwise. Every optimal policy routes type $t$ to
$\arg\max_i[q_i(t)-\lambda_B c_i]$, a per-type \emph{bang-per-buck} rule, mixing only where the budget binds. The value
$V(B)$ is nondecreasing, concave, and piecewise linear, with $V'(B)=\lambda_B$ wherever differentiable; once $B$ exceeds
the oracle's cost, $\lambda_B=0$ and the rule is the unconstrained per-query oracle.
\end{proposition}
\noindent This specializes standard linear-programming duality \citep{bertsimas1997lp} to inference routing. The step
is small but earns the allocation framing: orchestration under a budget is a priced assignment, and $\lambda_B$, the
\emph{shadow price of the inference dollar}, is the price a budget-aware buyer trades quality against (the same dollars-per-correct
currency the cost-aware rules of \S\ref{sec:B}--\ref{sec:C} optimize). With several resources (e.g.\ dollars and latency) $\lambda_B$ becomes a vector and the score is
$q_i(t)-\lambda_B\!\cdot c_i$. We verified the dual price, the bang-per-buck characterization, and concavity numerically.

\noindent\textbf{Dispersion scaling (heuristic).} If, purely as an illustration, the cell qualities $q_i(t)$ were i.i.d.\
$\mathcal N(\mu,s^2)$ across the $m\times|T|$ cells with $|T|$ large relative to $\ln m$, then $V^o\approx\mu+s\sqrt{2\ln m}$
while the best row mean concentrates at $\mu$, giving $G\sim s\sqrt{2\ln m}$: linear in cross-model within-type dispersion
$s$, only logarithmic in pool size $m$. We flag this as heuristic: real errors are block-correlated (\S\ref{sec:results}),
which lowers the rate, but the qualitative message (value from \emph{how} models differ, not how many) is what the
empirics test. A learned router with expected routing regret $R$ attains $V^o-R$ and beats single-best iff $R<G$. As $G$
is monotone in partition fineness, we report a partition-free per-query oracle as the true upper bound.

\subsection{The cost-aware diversification limit}\label{sec:B}
Let model errors satisfy $\Var(e_i)=\sigma^2$ and $\Corr(e_i,e_j)=\rho$, with systematic variance $\tau^2=\rho\sigma^2$
and idiosyncratic variance $(1-\rho)\sigma^2$.

\begin{proposition}[Variance floor (classical)]\label{prop:floor}
Equal-weight fusion of $k$ equicorrelated models has mean-squared error $V(k)=\sigma^2(\rho+\tfrac{1-\rho}{k})\to\rho\sigma^2$.
Under a symmetric single-factor probit with per-model error rate $\alpha$, unweighted majority vote has infinite-ensemble
error floor $\Phi(-\Phi^{-1}(1-\alpha)/\sqrt\rho)$, with limits $0$ ($\rho\to0$) and $\alpha$ ($\rho\to1$).
\end{proposition}
\noindent The continuous floor is the equicorrelated portfolio-variance limit \citep{markowitz1952,statman1987} and the
bias--variance--covariance/diversity decomposition \citep{ueda1996,krogh1995ensemble,wood2023diversity}; the copula floor is
\citet{turkmen2026ensembleselection}. We claim none of these.

\begin{proposition}[Cost-aware break-even]\label{prop:keven}
Let $\lambda$ be the local sensitivity of expected correctness to fused-estimate variance for the operative decision rule,
$\lambda := -\,\partial(\text{expected correct})/\partial V$, evaluated near the operating point (a derived Jacobian for a
threshold rule, not a risk-aversion parameter; under strict risk neutrality $\lambda$ is the local curvature through which
second moments enter the first-moment objective). With per-model cost $c$, the largest index whose addition pays for itself is
\begin{equation}
k^\ast(\rho,c)=\tfrac12\Big(-1+\sqrt{1+4\lambda\sigma^2(1-\rho)/c}\Big),\qquad
k^\ast\sim\sqrt{\lambda\sigma^2(1-\rho)/c}.\label{eq:keven}
\end{equation}
Hence $\partial k^\ast/\partial\rho<0$, $\partial k^\ast/\partial c<0$, and $\rho\to1\Rightarrow k^\ast\to0$. The optimal
ensemble cardinality is the nearest feasible integer to $k^\ast{+}1$.
\end{proposition}

\begin{proposition}[Calibration of $\lambda$ and a sign condition]\label{prop:lambda}
Under the symmetric single-factor probit of Prop.~\ref{prop:floor} (latent $S_i=m+\sqrt\rho\,U+\sqrt{1-\rho}\,\xi_i$,
$m=\Phi^{-1}(1-\alpha)$, correct iff $S_i>0$), equal-weight fusion has fused score $Z_k\sim\mathcal N(m,V(k))$, and the
risk-neutral majority/mean-vote rule has expected correctness $P(V)=\Phi(m/\sqrt V)$ in the single sufficient statistic
$V=V(k)$. The sensitivity $\lambda$ in \eqref{eq:keven} is then not a free parameter but the explicit Jacobian
\[ \lambda(V)=-\frac{\partial P}{\partial V}=\frac{m}{2V^{3/2}}\,\varphi\!\Big(\frac{m}{\sqrt V}\Big),\qquad m=\Phi^{-1}(1-\alpha), \]
evaluated at the operating point $V=V(k)$. Its sign is $\operatorname{sgn}\lambda=\operatorname{sgn}(\tfrac12-\alpha)$: when
each member beats chance ($\alpha<\tfrac12$), $\lambda>0$ and $P$ strictly decreases in $V$, so reducing $V$ by adding members
strictly raises expected correctness and $k^\ast$ is well defined; at $\alpha>\tfrac12$, $\lambda<0$ and the diversification
framing fails (adding members drives the majority \emph{more} wrong, and \eqref{eq:keven} returns no real $k^\ast$).
\end{proposition}
\noindent This pins down the economic bridge the cost-aware rule rests on and delimits where it is valid. $\lambda$ is an
operating-point Jacobian (it depends on $V$, hence weakly on $k$); we evaluate it at the marginal index, where the
local-linearization error $O(V(k)^{-2})$ is negligible. A finite-sample estimator plugs the tetrachoric $\hat\rho$ and per-model
$\hat\alpha$ into $\lambda(V(k))$; the Pearson coefficient of $0/1$ indicators biases $\rho$ (hence $\lambda$) downward. We
verified $\lambda=-\partial P/\partial V$ and the sign condition numerically.

\paragraph{Query-conditional $\rho$ and block covariance.} Classical $\rho$ is a global constant; making $\rho(t)$
type-conditional makes $k^\ast$ vary within a single workload and resolves the diversity-unreliability result
\citep{kuncheva2003} by conditioning. Because errors are block-structured and rise with accuracy \citep{kim2025correlated},
the headline object is the block covariance $\Sigma$: equal weights are no longer optimal (the minimum-variance weights are
$w\propto\Sigma^{-1}\mathbf 1$), and the floor is the undiversifiable common-factor component of $\Sigma$. The equicorrelation
form is the special case. We caution that the portfolio analogy is a first-moment-plus-local-curvature analogy, not a claim of
risk aversion.

\begin{figure}[tbp]\centering
\includegraphics[width=0.55\textwidth]{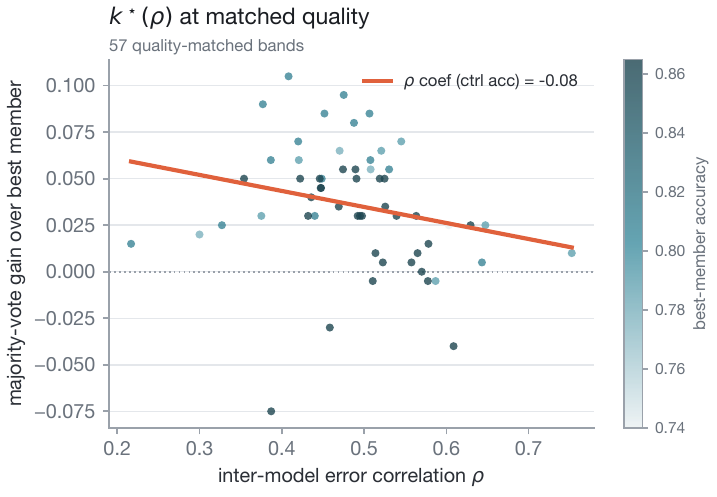}
\caption{\label{fig:kstar}Pillar B, $k^\ast(\rho)$ at matched quality: across 57 sub-bands of the matched 6-model pool (MMLU-Pro),
majority-vote gain over the best member vs.\ inter-model $\rho$, controlling for member accuracy. The slope is negative
(diversification-limit direction; CI spans zero), reversing the positive slope in the unmatched pool (Fig.~\ref{fig:rho}).}
\end{figure}

\subsection{Cascade calibration economics}\label{sec:C}
A cheap model $L$ (price $c_L$, accuracy $a_L$) answers first with confidence $s$; the query escalates to a strong model $H$
when $s<\tau$. Let $\beta=\Pr[s<\tau]$ be the escalation budget, $w(\beta)=\Pr[Y_L=0\mid\text{escalated }\beta\text{-tail}]$
the tail error rate, and $a_H(\tau)$ the strong model's accuracy on the deferred tail. Then $C(\beta)=c_L+\beta c_H$ and
$Q(\beta)=a_L+\beta[a_H(\tau)-1+w(\beta)]$.

\begin{proposition}[Collapse, ceiling, dominance]\label{prop:cascade}
Against random mixing (escalate a random $\beta$-fraction; tail error $1-a_L$, tail accuracy $a_H$) at equal budget,
\begin{equation}
Q(\beta)-Q_{\mathrm{mix}}(\beta)=\beta\big[w(\beta)-(1-a_L)\big]+\beta\big[a_H(\tau)-a_H\big].\label{eq:collapse}
\end{equation}
If $a_H(\tau)=a_H$ (no tail adverse selection), the integrated advantage $\int_0^1[w(\beta)-(1-a_L)]\,d\beta\ge 0$ iff the
verifier $\mathrm{AUC}\ge\tfrac12$, with equality iff $\mathrm{AUC}=\tfrac12$ (the cascade is then random mixing); the
pointwise inequality additionally requires $w$ non-increasing (a threshold-rational verifier). Under perfect calibration the
escalation ratio $\beta_{\mathrm{cas}}/\beta_{\mathrm{mix}}\to 1-a_L/a_H$: relative to random mixing, a calibrated cascade
needs only a fraction $1-a_L/a_H$ as many strong-model calls at a fixed quality floor, independent of price. Within a
one-parameter family of verifiers ordered by AUC, $\beta_{\mathrm{cas}}(q)$ is decreasing in AUC, so a critical
$\mathrm{AUC}^\ast$ exists below which (given $c_L/c_H<q/a_H$) no threshold meets floor $q\in(a_L,a_H]$ at lower
dollars-per-correct than the strong model alone.
\end{proposition}

\begin{corollary}[Calibration-and-edge: necessary condition]\label{cor:edge}
Cascade dominance over $H$-only requires both $\mathrm{AUC}(s)>\mathrm{AUC}^\ast$ and a positive conditional edge
$a_H(\tau)>a_L(\tau)$ on the deferred tail. If $H$ is adversely weak where $L$ defers, calibration alone is insufficient
(consistent with \citealp{jitkrittum2023cascade}).
\end{corollary}


\section{Proofs}\label{app:proofs}
\begin{proof}[Proof of Lemma~\ref{lem:env}]
For any $\pi$ and $t$, $\sum_i\pi(i\mid t)q_i(t)\le\max_i q_i(t)$; take $\E_t$ to get $V(\pi)\le V^o$, attained by $\pi^o$. Write
$G=\E_t[\max_i q_i(t)-q_{i^\ast}(t)]$ with $i^\ast=\arg\max_i\bar q_i$; the integrand is nonnegative. If model $j$ is uniformly best
$D$-a.e.\ then $\bar q_j=V^o$, so $G=0$; conversely, if no model is uniformly best, the set where $i^\ast$ is suboptimal has positive
measure and the integrand is strictly positive there, so $G>0$. The dispersion-scaling remark uses the extreme-value asymptotic
$\E[\max$ of $m$ i.i.d.\ $\mathcal N(0,1)]\sim\sqrt{2\ln m}$ and concentration of the best row mean at $\mu$ when $|T|\gg\ln m$; it is
heuristic under correlation.
\end{proof}
\begin{proof}[Proof of Prop.~\ref{prop:floor}]
With $e_i=b+\varepsilon_i$, $\Var(b)=\rho\sigma^2$, $\Var(\varepsilon_i)=(1-\rho)\sigma^2$ uncorrelated, the equal-weight average has
$\Var(\bar e)=\rho\sigma^2+(1-\rho)\sigma^2/k\to\rho\sigma^2$. For binary correctness under $S_i=m+\sqrt\rho\,U+\sqrt{1-\rho}\,\xi_i$ with
$m=\Phi^{-1}(1-\alpha)$, conditioning on $U$ gives correctness probability $\Phi((m+\sqrt\rho U)/\sqrt{1-\rho})$; by the law of large
numbers the majority is correct iff $U>-m/\sqrt\rho$, so the floor is $\Phi(-m/\sqrt\rho)$; the limits are immediate.
\end{proof}
\begin{proof}[Proof of Prop.~\ref{prop:keven}]
The marginal variance reduction of the $(k{+}1)$-th model is $\Delta V(k)=\sigma^2(1-\rho)/[k(k+1)]$. Converting to expected correctness
through the local sensitivity $\lambda$, adding a model is worthwhile iff $\lambda\,\Delta V(k)\ge c$, i.e.\ $k(k+1)\le R$ with
$R=\lambda\sigma^2(1-\rho)/c$; solving $k(k+1)=R$ gives \eqref{eq:keven}. The comparative statics follow by inspection; $R\to0$ as
$\rho\to1$. $k^\ast$ is the last paying index, so the optimal cardinality is $\approx k^\ast+1$.
\end{proof}
\begin{proof}[Proof of Prop.~\ref{prop:cascade} and Cor.~\ref{cor:edge}]
Escalated mass $\beta$ pays $c_H$ on top of $c_L$, giving $C(\beta)=c_L+\beta c_H$; of $L$'s correct mass $a_L$, the part in the escalated
tail $\beta(1-w(\beta))$ is replaced by $H$ scoring $a_H(\tau)$ there, giving $Q(\beta)=a_L+\beta[a_H(\tau)-1+w(\beta)]$. Random mixing
escalates a random $\beta$-fraction with tail error $1-a_L$ and tail accuracy $a_H$, giving $Q_{\mathrm{mix}}=a_L+\beta(a_H-a_L)$;
subtracting yields \eqref{eq:collapse}. With $a_H(\tau)=a_H$, the residual is $\beta[w(\beta)-(1-a_L)]$, whose integral is nonnegative iff
$s$ ranks errors above successes on average ($\mathrm{AUC}\ge\tfrac12$), with equality iff $\mathrm{AUC}=\tfrac12$; the pointwise sign
needs $w$ non-increasing. Under perfect calibration the escalated tail is exactly $L$'s errors until exhausted, so $Q(\beta)=a_L+\beta a_H$
for $\beta\le 1-a_L$, whence $\beta_{\mathrm{cas}}(q)=(q-a_L)/a_H$ and $\beta_{\mathrm{cas}}/\beta_{\mathrm{mix}}\to 1-a_L/a_H$. Beating
$H$-only ($c_H/a_H$ per correct) at floor $q$ requires $c_L+\beta_{\mathrm{cas}}(q)c_H\le(q/a_H)c_H$, feasible only if $c_L/c_H<q/a_H$;
within a one-parameter AUC-ordered family $\beta_{\mathrm{cas}}$ decreases in AUC, so $\mathrm{AUC}^\ast$ exists. If $a_H(\tau)\le a_L(\tau)$
on the deferred tail, escalation cannot raise quality there, so no $\tau$ achieves $q$: dominance requires both AUC above threshold and a
positive tail edge. We prove necessity; sufficiency additionally needs reachability of $q$ and the cost ordering.
\end{proof}
\begin{proof}[Proof of Prop.~\ref{prop:option}]
The captured value per arrival is $v\,\E[\max(\Gamma,0)]=v\,\mathrm{Eg}$ by the Gaussian truncation identity. A captured lead is held for an
$\mathrm{Exp}(\nu)$ time and discounted at $r$ (factor $1/(r+\nu)$); the Poisson arrival stream contributes a present-value
multiplier $\nu/r$, giving $V_R=v\,\mathrm{Eg}\cdot\nu/(r+\nu)\cdot(1/r)$. Then $\partial\mathrm{Eg}/\partial\eta=\varphi(\mu/\eta)>0$
and $\partial[\nu/(r+\nu)]/\partial\nu>0$; adding a separable $m(\text{level})$ leaves both partials unchanged.
\end{proof}
\begin{proof}[Proof of Prop.~\ref{prop:duality}]
$(P_B)$ is linear in $\{\pi(i\mid t)\}$. Dualize only the budget row with multiplier $\lambda\ge0$, keeping the per-type
simplices as the domain; the Lagrangian separates across $t$ into inner problems
$\max_{\pi(\cdot\mid t)\in\Delta}\sum_i\pi(i\mid t)(q_i(t)-\lambda c_i)=\max_i(q_i(t)-\lambda c_i)=:g_t(\lambda)$. Weak duality
gives $V(B)\le\lambda B+\E_t g_t(\lambda)$ for every $\lambda\ge0$; LP strong duality (Slater holds for $B>\underline B$) gives
equality at some $\lambda_B$. Complementary slackness places every optimal $\pi^\star(\cdot\mid t)$ on
$\arg\max_i(q_i(t)-\lambda_B c_i)$, mixing only to exhaust the budget. As a minimum of affine functions of $B$, $V$ is concave;
it is nondecreasing (enlarging $B$ relaxes a constraint) and piecewise linear (parametric LP). By the envelope theorem
$V'(B)=\lambda_B$ where differentiable, with subdifferential $[\lambda_B^-,\lambda_B^+]$ at the finitely many kinks; once
$B\ge K(\pi^o)$ the budget is slack, $\lambda_B=0$, and the rule is the per-query oracle.
\end{proof}
\begin{proof}[Proof of Prop.~\ref{prop:lambda}]
Under the single-factor probit, $Z_k=\tfrac1k\sum_i S_i=m+\sqrt\rho\,U+\tfrac{\sqrt{1-\rho}}{k}\sum_i\xi_i\sim\mathcal N(m,V(k))$,
$V(k)=\rho+(1-\rho)/k$ ($\sigma^2{=}1$). The mean-vote rule is correct iff $Z_k>0$, so $P(V)=\Pr[Z_k>0]=\Phi(m/\sqrt V)$.
Then $\partial P/\partial V=\varphi(m/\sqrt V)\cdot m\cdot(-\tfrac12)V^{-3/2}$, so $\lambda:=-\partial P/\partial V=(m/2V^{3/2})\varphi(m/\sqrt V)$.
As $\varphi>0,V>0$, $\operatorname{sgn}\lambda=\operatorname{sgn}m=\operatorname{sgn}(\Phi^{-1}(1-\alpha))=\operatorname{sgn}(\tfrac12-\alpha)$.
Substituting $\lambda$ into $\lambda\,\Delta V(k)\ge c$ with $\Delta V(k)=(1-\rho)/[k(k+1)]$ reproduces \eqref{eq:keven}.
\end{proof}
\begin{proof}[Proof of Prop.~\ref{prop:cert}]
(i) On the all-wrong event every member answer differs from the label; a selection policy outputs one member answer, hence is
wrong there, so $\mathrm{Acc}(\pi)\le\Pr[\text{not all wrong}]=1-\beta$, attained by the per-query oracle. (ii)
$V^o=\E_t\mathbf 1\{\exists i:Y_i{=}1\}=1-\beta$ and $a_{\mathrm{sb}}=\Pr[Y_{i^\star}{=}1]$, so
$G=V^o-a_{\mathrm{sb}}=\Pr[Y_{i^\star}{=}0]-\beta=\Pr[\text{single-best wrong}]-\beta$, the non-unanimous single-best-wrong mass.
(iii) $K=\sum_j\mathbf 1\{\text{all wrong on }t_j\}\sim\mathrm{Binomial}(n,\beta)$ since the $t_j$ are i.i.d.\ and ``all wrong''
is a fixed event; the Clopper--Pearson lower limit obeys $\Pr[\beta\ge\beta_{\mathrm{lo}}]\ge1-\delta$, so with probability
$\ge1-\delta$, $(1-\beta)\le(1-\beta_{\mathrm{lo}})$ and every selection policy obeys
$\mathrm{Acc}-a_{\mathrm{sb}}\le(1-\beta)-a_{\mathrm{sb}}\le(1-\beta_{\mathrm{lo}})-a_{\mathrm{sb}}$. A union bound replaces
$a_{\mathrm{sb}}$ by an upper confidence bound at level $\delta'$, with total error $\le\delta+\delta'$.
\end{proof}
\begin{proof}[Proof of Prop.~\ref{prop:poolbias}]
$\beta(m)=\pi+(1-\pi)\alpha_0^{\,m}$ is strictly decreasing to $\pi$ since $0<\alpha_0<1$. Calibrate the single-factor Gaussian
copula to reproduce the marginal $\alpha$ and the pairwise co-error $\pi+(1-\pi)\alpha_0^2$, fixing a latent correlation
$\bar\rho<1$; then $\beta_{\mathrm{sf}}(2)=\pi+(1-\pi)\alpha_0^2=\beta(2)$, so the ratio is $1$ at $m=2$. For $m\ge2$, conditioning
on the common factor $Z$, $\beta_{\mathrm{sf}}(m)=\int\varphi(z)\,\Phi\!\big((t-\sqrt{\bar\rho}\,z)/\sqrt{1-\bar\rho}\big)^m\,dz$ with
$t=\Phi^{-1}(\alpha)$. For each fixed $z$ the conditional error probability $\Phi(\cdot)\in(0,1)$ since $\bar\rho<1$, so the
integrand $\to0$ pointwise; dominated convergence gives $\beta_{\mathrm{sf}}(m)\to0$ (this is the zero lower-tail-dependence of the
Gaussian copula, \citealp{embrechts2002correlation}). Since $\beta(m)\to\pi>0$ and $\beta_{\mathrm{sf}}(m)\to0$, the ratio
$\beta(m)/\beta_{\mathrm{sf}}(m)\to\infty$. For monotonicity, compare successive decay rates: $\beta(m{+}1)/\beta(m)\to1$ because
$\beta(m)\to\pi>0$, whereas $\beta_{\mathrm{sf}}(m{+}1)/\beta_{\mathrm{sf}}(m)$ tends to the conditional error probability at the
dominating factor value, which is $<1$ for $\bar\rho<1$; hence the ratio's successive increments are eventually positive and the
underpricing ratio is eventually strictly increasing (numerically, strict from $m{=}2$ in all settings we tested).
\end{proof}
\begin{proof}[Proof of Prop.~\ref{prop:nonid}]
Take exchangeable error triples on $\{0,1\}^3$ and let $q_j$ be the probability of exactly $j$ errors ($\sum_j q_j=1$). Every
pair has the same $2{\times}2$ table, determined by the marginal error rate $p_1=\tfrac13\sum_j j\,q_j$ and the pairwise
co-error $p_2=\Pr[X_i{=}X_j{=}1]=\tfrac13(q_2+3q_3)$; the triple-failure cell $\beta=q_3$ is left free in the resulting
one-parameter Fr\'echet class. Concretely, fix $p_1=\tfrac12,\ p_2=\tfrac14$: both $(q_0,q_1,q_2,q_3)=(\tfrac14,0,\tfrac34,0)$
and $(0,\tfrac34,0,\tfrac14)$ reproduce \emph{every} one- and two-dimensional marginal---hence every pairwise Pearson and
tetrachoric correlation, which are functions of that common bivariate table---yet have $\beta=0$ and $\beta=\tfrac14$. So
$\beta$ is not a function of the pairwise law. Padding with independent coordinates extends the example to any $m\ge3$, and the
common-shock mixture ($\beta_\infty=\pi>0$) versus a matched-$\bar\rho$ Gaussian copula ($\beta_\infty=0$ by zero lower
tail dependence) realize the two extremes as $m\to\infty$, both consistent with the same pairwise $\bar\rho$.
\end{proof}

\section{Reproducibility}\label{app:repro}
\textbf{Data and code availability.} We release the full per-cell outcome matrices (every model's graded correctness, token
counts, and metered cost on every query), the model registry with dated snapshots and live prices, the programmatic graders, and
all analysis scripts, so that every number in the paper---including the $\beta$/$\rho$ tables, the pool-size scaling, and the
two-regime comparison---can be regenerated end to end. \textbf{Runnable certificate.} The realizability certificate
(Prop.~\ref{prop:cert}) ships as a standalone tool, \texttt{beta\_certificate.py}: from a pool's logged outcomes---or merely the
all-wrong count $K/n$ and the single-best accuracy---it returns the Clopper--Pearson-certified \$0 lower bound on the maximum
gain any router, vote, or cascade could deliver over single-best, requiring no pairwise $\rho$ (Prop.~\ref{prop:nonid}). The
full-$\Sigma$ residual decomposition (\texttt{residual\_decomp.py}) and the propagated ratio CIs (\texttt{bootstrap\_ratio.py})
are likewise released. All runs are logged with exact model identifiers, dated snapshots, prices,
seeds, and per-call costs. \textbf{Spend.} We meter each call against the OpenRouter account usage endpoint (an aggregator, not a
per-provider source); per-run figures are sums of logged per-cell costs (cache-independent), itemized in the released
\texttt{cost\_registry.csv}. The core pillar experiments (A--D) cost $\approx$\$47 (matching the dated \texttt{/key} ground-truth
total); the market-scale measurement adds $\approx$\$111 (the 53- and refreshed 67-model realizability runs, the truncation re-runs, and the
GPQA second-domain run), and the third-domain experiments---execution-graded code (App.~\ref{app:codegen}) and the open-ended
GPQA panel (App.~\ref{app:opengpqa})---add $\approx$\$110, for $\approx$\$270 of reported-experiment cost. Total account usage
including all exploratory and superseded iteration was $\approx$\$560 at submission (the live usage meter was transiently
clobbered by concurrent runs late in the project; this figure is reconstructed from the last clean reading plus the logged
third-domain run costs, and we flag it as approximate). We never present account-level usage as experiment cost. Programmatic graders and the registry are versioned;
illustrative parameter values are labelled and never reported as measurements. \textbf{Provenance.} The market-scale matrix is
reconstructed from the response cache (\texttt{reconstruct.py}) after a hard-first run was stopped at $\approx$41\% (the saturated
benchmarks, least informative for the tail, were not needed); the canonical graded, truncation-corrected matrix is
\texttt{matrix\_marketE2} (67-model frontier refresh). The empirical claims correspond to run identifiers \texttt{matrix\_stageA2},
\texttt{matrix\_hardA}, \texttt{matrix\_churnD} (App.~\ref{app:churn} churn study), \texttt{fusion\_eqq2} (matched quality), and the market-scale
reconstruction. The tail-co-failure numbers regenerate via \texttt{realizability.py --tag marketE2}. \emph{Selection used for the headline:} the
MATH-500 figures ($\beta=0.052$, tetrachoric underpricing $2.5\times$, $k{=}17$, $n{=}330$; \texttt{realizability\_tetrachoric.json}) are computed on the common-coverage subset of all $67$ models (queries
answered by every model, run id \texttt{matrix\_marketE3}); GPQA on its $52$-model complete-coverage subset
(\texttt{matrix\_marketE2}). \texttt{realizability.py}'s default $\ge0.95$-coverage filter surfaces a slightly smaller,
also-valid model subset with the same order-of-magnitude underpricing; both are correct. The truncation
control via \texttt{detruncate.py}; the cascade collapse (20 seeds) via \texttt{cascade.py}; the matched-quality
partition-robustness ($+0.027$, positive in all $60$ partitions) via \texttt{eqq\_robustness.py}; and Pillar-B clustered inference
via the model-jackknife in \texttt{rho\_fusion\_test.py}.
All references were audited with scite Smart Citations: every entry resolves to a real indexed work, none carries a retraction,
correction, or editorial concern, and no load-bearing citation has any contrasting smart-citations (the only contrasting
citations in the bibliography are two on \citet{statman1987}, from the unrelated optimal-holding-count debate).

\section{The market-scale model pool}\label{app:roster}
The $67$-model, $21$-family market pool used in \S\ref{sec:realiz} and the two-regime measurement, with live OpenRouter prices
(snapshot 2026-06-19). All are chat/instruct models; pure reasoning/``thinking'' variants are excluded so finite-token
programmatic grading is clean. Prices are USD per million tokens.
\begin{center}\scriptsize
\begin{longtable}{llrrr}
\toprule
model & family & tier & \$/Mtok in & \$/Mtok out \\
\midrule
\endfirsthead
\toprule model & family & tier & \$/Mtok in & \$/Mtok out \\ \midrule \endhead
\bottomrule \endfoot
\input{roster_rows}
\end{longtable}
\end{center}

\section{Optionality under churn (secondary)}\label{app:churn}
Deferred from the main thread because it is observational and not load-bearing for the $\beta$/$\rho$ result. Frontier releases
arrive as a Poisson process of intensity $\nu$ (the churn rate); each offers a relative capability gap $\Gamma=\mu+\eta Z$,
$Z\sim\mathcal N(0,1)$. Broad access (ROUTE) captures $v\max(\Gamma,0)$ per arrival, held until the next arrival and discounted at
$r$; committing (self-host) earns premium $\delta$ but pays switch cost $K$.

\begin{proposition}[Additive option value]\label{prop:option}
With $\mathrm{Eg}=\mu\Phi(\mu/\eta)+\eta\,\varphi(\mu/\eta)$ ($\varphi$ the standard normal density), the option value of broad
access is $V_R=\tfrac{v}{r}\cdot\tfrac{\nu}{r+\nu}\cdot\mathrm{Eg}(\mu,\eta)$, and the build-vs-route threshold is
$\delta^\ast=rK+v\tfrac{\nu}{r+\nu}\mathrm{Eg}$. If buyers additionally value absolute capability via a separable
$m(\cdot)$ independent of $(\nu,\eta)$, total value is $m(\text{level})+V_R$, so $\partial V/\partial\nu>0$ and
$\partial V/\partial\eta>0$ regardless of $dm/d\text{level}$.
\end{proposition}

\begin{corollary}[Convergence collapse]\label{cor:converge}
If error correlation rises toward $1$ as models improve (an extrapolation of \citet{kim2025correlated}, who document rising,
not unit, correlation), the diversification gain shrinks; jointly with a falling common error rate, the orchestration value $G$
and $V_R$ contract. We verify the premise's sign, not the unit limit.
\end{corollary}
\noindent The comparative statics are standard real-options results \citep{mcdonald1986,dixit1994,merton1976}; the separability of
$m$ is an assumption, not a finding, and the contribution is only the additive form under it. \textbf{Measured on a real timeline:}
on the release sequence of an 18-model generational pool (Claude-3-Haiku Mar 2024 to Gemini-3.1-Pro Feb 2026), the best achievable
dollars-per-correct dropped $\approx\!14$--$15\times$ (clean, since cost is metered not graded), echoing \citet{erol2025costofpass};
and the capability option value of broad access is regime-dependent---$+0.33$ accuracy by 2026 on hard MMLU-Pro versus $+0.01$ on
saturated GSM8K. This is an observational single-path study, not a controlled comparative static; the dispersion$\to$capture
association is directionally positive but noisy, so the grading-independent cost-churn result is the robust part.

\begin{figure}[tbp]\centering
\includegraphics[width=0.55\textwidth]{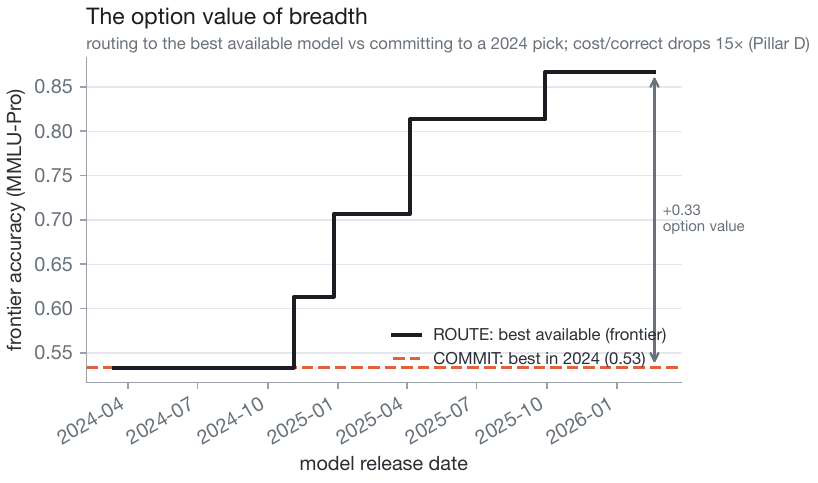}
\caption{\label{fig:churn}Optionality under churn (secondary). Frontier accuracy on MMLU-Pro over 2024--2026 (ROUTE: adopt each
new best) vs.\ committing to the best 2024 model: a $+0.33$ realized option value of breadth where headroom exists, alongside a
$\approx\!14\times$ fall in frontier dollars-per-correct.}
\end{figure}

\section{The third domain: execution-graded competitive code}\label{app:codegen}
The binding external-validity question is whether co-failure is a math artifact or tracks open-ended generation. We test it on a
\emph{structurally independent} domain, competitive programming, where generation is open-ended but grading is programmatic and
the task family is disjoint from math.

\paragraph{Construction and strict grader.} From \texttt{deepmind/code\_contests} (Codeforces rating $1900$--$3500$) we take
problems and retain only those whose \emph{own} accepted Python-3 reference our grader accepts (so an all-wrong event is genuine
co-failure, not a grader artifact); $63$ problems pass (of $140$ fetched; $77$ dropped, mostly for lacking a short Python-3
reference), each with a median of $20$ tests. Each model writes a stdin/stdout
program; we execute it against the problem's \emph{private} $+$ \emph{generated} stress tests (not the tiny public samples),
inside a memory-capped sandbox, enforcing $3\times$ the official (C++-calibrated) time limit---strict, but \emph{Python-fair}, so
that correct-but-slower Python is not failed (which would manufacture co-failure). $18$ frontier models $\times$ $63$ problems
spanning rating $1900$--$3500$ (the harder tail, where frontier co-failure is more frequent).

\paragraph{Result: the co-failure signature replicates and resolves.} Mean accuracy is $0.45$ (these are genuinely hard); the tail
is populated, $\beta=0.079$ ($k{=}5$ all-wrong over $63$, CP$[0.026,0.176]$). The whole math pattern recurs: the naive
Pearson-of-indicators calibration ($\bar\rho{=}0.27$) implies a spurious $17\times$, the correct tetrachoric calibration
($\bar\rho_{\mathrm{tet}}{=}0.51$) gives a $\mathbf{3.1\times}$ underpricing, and the full $\Sigma$ Gaussian copula still leaves a
$1.7\times$ residual (\texttt{residual\_decomp.json})---the same common-mode signature, in a domain sharing nothing with
competition math but open-endedness. With the harder problems added, the underpricing is now \emph{statistically resolved}:
bootstrap $90\%$ CI $[1.5,6.2]$, excluding $1$ (\texttt{ratio\_uncertainty.json}).

\paragraph{What we do and do not claim.} The cross-domain \emph{regime} is established: $\beta>0$ with the Pearson trap, a
full-$\Sigma$ residual, and now an underpricing ratio whose CI excludes $1$, in a third disjoint open-ended domain, versus
$\beta\approx0$ on multiple-choice GPQA. We are still cautious about the precise code \emph{magnitude}: $k{=}5$ is a small event
base (the ratio point moves within $[1.5,6.2]$), and two limits remain---(i)~the grader uses real private$+$generated stress tests
under a Python-fair time limit but is \emph{not} the official online judge (the dataset API may truncate the largest
$n,q\sim10^5$ inputs), and (ii)~$\beta$ rests on $18$ models, and the market pool logs no code prompts so no in-domain router was
trained. A precise code magnitude at $67$-model scale via the official hidden-test judge is the one remaining sharpening; the
cross-domain regime and a significant code underpricing are established by the data above.

\section{Content-controlled format test: open-ended GPQA (LLM-judge panel)}\label{app:opengpqa}
To rule out that the open-ended-versus-multiple-choice regime split is a \emph{content} confound (math/code happen to be
open-ended; science happens to be multiple-choice), we hold content fixed and vary only format. We take the \emph{same}
GPQA-Diamond questions used in the multiple-choice analysis, strip the options, and ask each of the $18$ frontier models for a
free-response answer (\texttt{data.py:gpqa\_open}).

\paragraph{Grading by an LLM-judge panel.} With no programmatic oracle, each (question, answer) is graded for equivalence to the
reference answer by a panel of five direct-answering judges (Claude-Sonnet-4.6, DeepSeek-V3.2, Gemini-3.5-Flash, Qwen3.7-Max,
Mistral-Large-2512), majority vote, with a judge excluded from grading its own model's answer (\texttt{judge\_open.py}); ties
break to \emph{incorrect}, a conservative-against-accuracy choice. We quantify the grader's reliability rather
than assume it: pairwise inter-judge agreement across the five judges is $\kappa=0.73$--$0.92$ (Cohen's $\kappa$,
substantial-to-near-perfect, over $10$ judge pairs). Reasoning models that exhausted the $2048$-token budget on hidden
reasoning and returned an empty answer were re-queried at $8192$ then $16384$ tokens (de-truncation), so an empty cell is not
scored as a false failure.

\paragraph{Result: changing only the format flips the regime.} On a matched comparison---the fully-judged questions and the
$11$ models present in \emph{both} the multiple-choice and open runs---mean accuracy falls $0.66\!\to\!0.51$ and best-model
accuracy $0.91\!\to\!0.77$ (so the collapse is not an artifact of differing pools or aggregations). The co-failure tail
\emph{opens}: $\beta=0.127$ ($k{=}10$ all-wrong, CP$[0.062,0.220]$) over the $79$ questions with complete $18$-model coverage,
versus $\beta\approx0$ on the multiple-choice version of the identical questions. The tail is \textbf{robust to judge
aggregation}: requiring \emph{unanimous} judge agreement for a correct mark gives $\beta=0.241$ ($k{=}19$), and the most
\emph{lenient} rule (any one of five judges calls it correct) still gives $\beta=0.038$ ($k{=}3$)---positive under every
aggregation, where the multiple-choice version is $\approx0$ throughout. Open
GPQA thus joins MATH and code in the ceiling-bound regime, with content held fixed---decisive evidence that \emph{open-endedness},
not subject matter, drives co-failure, and that it does so even under an LLM judge, not only programmatic grading.

\paragraph{Honest limits.} Coverage is partial: $79$ of $130$ questions reach complete $18$-model coverage (de-truncation raised
this from $29\!\to\!55\!\to\!79$; the rest still lose $\ge1$ model to truncation/refusal on the hardest items, so the true $\beta$
is plausibly \emph{higher}---verified: the dropped items have far lower partial-coverage accuracy, so they would add co-failure
events). Grading is by a $5$-judge LLM panel ($\kappa\,0.73$--$0.92$), not human adjudication---a human-calibrated judge on the
full $130$ remains the clean follow-up. The \emph{direction, magnitude, and regime flip} are unambiguous and content-controlled.
\end{document}

%% file: roster_rows.tex
\texttt{\scriptsize openai/gpt-5.5} & openai & frontier & 5 & 30 \\
\texttt{\scriptsize anthropic/claude-opus-4.8} & anthropic & frontier & 5 & 25 \\
\texttt{\scriptsize anthropic/claude-sonnet-4.6} & anthropic & frontier & 3 & 15 \\
\texttt{\scriptsize openai/gpt-5.4} & openai & frontier & 2.5 & 15 \\
\texttt{\scriptsize openai/gpt-5.2} & openai & frontier & 1.75 & 14 \\
\texttt{\scriptsize google/gemini-3.1-pro-preview} & google & frontier & 2 & 12 \\
\texttt{\scriptsize openai/gpt-5.1} & openai & mid & 1.25 & 10 \\
\texttt{\scriptsize google/gemini-3.5-flash} & google & mid & 1.5 & 9 \\
\texttt{\scriptsize ai21/jamba-large-1.7} & ai21 & mid & 2 & 8 \\
\texttt{\scriptsize mistralai/mistral-medium-3-5} & mistral & mid & 1.5 & 7.5 \\
\texttt{\scriptsize qwen/qwen3.6-max-preview} & qwen & mid & 1.04 & 6.24 \\
\texttt{\scriptsize writer/palmyra-x5} & writer & mid & 0.6 & 6 \\
\texttt{\scriptsize anthropic/claude-haiku-4.5} & anthropic & mid & 1 & 5 \\
\texttt{\scriptsize z-ai/glm-5.2} & zai & mid & 1.2 & 4.1 \\
\texttt{\scriptsize qwen/qwen3-max} & qwen & mid & 0.78 & 3.9 \\
\texttt{\scriptsize qwen/qwen3.7-max} & qwen & mid & 1.25 & 3.75 \\
\texttt{\scriptsize moonshotai/kimi-k2.7-code} & moonshot & mid & 0.74 & 3.5 \\
\texttt{\scriptsize qwen/qwen3-coder-plus} & qwen & mid & 0.65 & 3.25 \\
\texttt{\scriptsize z-ai/glm-5.1} & zai & mid & 0.98 & 3.08 \\
\texttt{\scriptsize nousresearch/hermes-4-405b} & nous & mid & 1 & 3 \\
\texttt{\scriptsize x-ai/grok-4.3} & xai & mid & 1.25 & 2.5 \\
\texttt{\scriptsize moonshotai/kimi-k2-0905} & moonshot & mid & 0.6 & 2.5 \\
\texttt{\scriptsize qwen/qwen3.5-397b-a17b} & qwen & mid & 0.385 & 2.45 \\
\texttt{\scriptsize nvidia/nemotron-3-ultra-550b-a55b} & nvidia & mid & 0.5 & 2.2 \\
\texttt{\scriptsize qwen/qwen3.5-122b-a10b} & qwen & mid & 0.26 & 2.08 \\
\texttt{\scriptsize openai/gpt-5-mini} & openai & mid & 0.25 & 2 \\
\texttt{\scriptsize z-ai/glm-5} & zai & mid & 0.6 & 1.92 \\
\texttt{\scriptsize z-ai/glm-4.7} & zai & mid & 0.4 & 1.75 \\
\texttt{\scriptsize z-ai/glm-4.6} & zai & mid & 0.43 & 1.74 \\
\texttt{\scriptsize mistralai/mistral-large-2512} & mistral & mid & 0.5 & 1.5 \\
\texttt{\scriptsize google/gemini-3.1-flash-lite} & google & mid & 0.25 & 1.5 \\
\texttt{\scriptsize qwen/qwen3.7-plus} & qwen & mid & 0.32 & 1.28 \\
\texttt{\scriptsize minimax/minimax-m3} & minimax & cheap & 0.3 & 1.2 \\
\texttt{\scriptsize stepfun/step-3.7-flash} & stepfun & cheap & 0.2 & 1.15 \\
\texttt{\scriptsize qwen/qwen3-next-80b-a3b-instruct} & qwen & cheap & 0.09 & 1.1 \\
\texttt{\scriptsize minimax/minimax-m2} & minimax & cheap & 0.255 & 1 \\
\texttt{\scriptsize minimax/minimax-m2.7} & minimax & cheap & 0.25 & 1 \\
\texttt{\scriptsize deepseek/deepseek-v3.1-terminus} & deepseek & cheap & 0.27 & 0.95 \\
\texttt{\scriptsize minimax/minimax-m2.5} & minimax & cheap & 0.15 & 0.9 \\
\texttt{\scriptsize deepseek/deepseek-v4-pro} & deepseek & cheap & 0.435 & 0.87 \\
\texttt{\scriptsize xiaomi/mimo-v2.5-pro} & xiaomi & cheap & 0.435 & 0.87 \\
\texttt{\scriptsize deepseek/deepseek-chat-v3.1} & deepseek & cheap & 0.21 & 0.79 \\
\texttt{\scriptsize qwen/qwen-plus-2025-07-28} & qwen & cheap & 0.26 & 0.78 \\
\texttt{\scriptsize google/gemma-2-27b-it} & google & cheap & 0.65 & 0.65 \\
\texttt{\scriptsize meta-llama/llama-4-maverick} & meta & cheap & 0.15 & 0.6 \\
\texttt{\scriptsize upstage/solar-pro-3} & upstage & cheap & 0.15 & 0.6 \\
\texttt{\scriptsize nvidia/nemotron-3-super-120b-a12b} & nvidia & cheap & 0.09 & 0.45 \\
\texttt{\scriptsize nvidia/llama-3.3-nemotron-super-49b-v1.5} & nvidia & cheap & 0.4 & 0.4 \\
\texttt{\scriptsize meta-llama/llama-3.1-70b-instruct} & meta & cheap & 0.4 & 0.4 \\
\texttt{\scriptsize nousresearch/hermes-4-70b} & nous & cheap & 0.13 & 0.4 \\
\texttt{\scriptsize openai/gpt-5-nano} & openai & cheap & 0.05 & 0.4 \\
\texttt{\scriptsize microsoft/phi-4-mini-instruct} & microsoft & cheap & 0.08 & 0.35 \\
\texttt{\scriptsize deepseek/deepseek-v3.2} & deepseek & cheap & 0.2288 & 0.3432 \\
\texttt{\scriptsize meta-llama/llama-3.2-3b-instruct} & meta & cheap & 0.0509 & 0.335 \\
\texttt{\scriptsize meta-llama/llama-3.3-70b-instruct} & meta & cheap & 0.1 & 0.32 \\
\texttt{\scriptsize nvidia/nemotron-3-nano-30b-a3b} & nvidia & cheap & 0.05 & 0.2 \\
\texttt{\scriptsize deepseek/deepseek-v4-flash} & deepseek & cheap & 0.09 & 0.18 \\
\texttt{\scriptsize openai/gpt-oss-120b} & openai & cheap & 0.039 & 0.18 \\
\texttt{\scriptsize openai/gpt-oss-20b} & openai & cheap & 0.029 & 0.14 \\
\texttt{\scriptsize google/gemma-3n-e4b-it} & google & cheap & 0.06 & 0.12 \\
\texttt{\scriptsize ibm-granite/granite-4.0-h-micro} & ibm & cheap & 0.017 & 0.112 \\
\texttt{\scriptsize qwen/qwen3-235b-a22b-2507} & qwen & cheap & 0.09 & 0.1 \\
\texttt{\scriptsize ibm-granite/granite-4.1-8b} & ibm & cheap & 0.05 & 0.1 \\
\texttt{\scriptsize mistralai/mistral-small-24b-instruct-2501} & mistral & cheap & 0.05 & 0.08 \\
\texttt{\scriptsize meta-llama/llama-3.1-8b-instruct} & meta & cheap & 0.02 & 0.03 \\
\texttt{\scriptsize mistralai/mistral-nemo} & mistral & cheap & 0.02 & 0.03 \\
\texttt{\scriptsize inclusionai/ling-2.6-flash} & inclusionai & cheap & 0.01 & 0.03 \\